\documentclass{article}

\usepackage{arxiv}

\usepackage[utf8]{inputenc}
\usepackage{subcaption}
\usepackage{booktabs}
\usepackage{bm}
\usepackage{bbm}
\usepackage{tcolorbox}
\usepackage{multicol,multirow}
\usepackage{natbib}
\usepackage{comment}
\usepackage{booktabs}

\usepackage{enumerate}   
\usepackage{amsmath}
\usepackage{tikz}

\usetikzlibrary{positioning,arrows}

\usepackage{wrapfig}
\DeclareMathOperator*{\argmax}{arg\,max}
\DeclareMathOperator*{\argmin}{arg\,min}

\usepackage{amsmath}
\usepackage{appendix}
\usepackage{amssymb}

\usepackage{amsthm}
\usepackage[capitalise]{cleveref}
\Crefname{assumption}{Assumption}{Assumptions}

\theoremstyle{plain}
\newtheorem{theorem}{Theorem}

\newtheorem{corollary}{Corollary}
\theoremstyle{definition}
\newtheorem{assumption}{Assumption}
\newtheorem{proposition}{Proposition}
\newtheorem{definition}{Definition}
\newtheorem{remark}{Remark}

\newcommand{\iid}{ \stackrel{i.i.d.}{\sim} }

\def\Holder{{H\"{o}lder}}

\def\Cramer{Cram\'{e}r}

\newcommand{\op}{\mathrm{o}_{p}}

\newcommand{\pa}{\mathrm{\pa}}

\newcommand{\epol}{\pi^\mathrm{e}}

\renewcommand{\eqref}[1]{(\ref{#1})}

\newcommand{\RN}[1]{%
  \textup{\uppercase\expandafter{\romannumeral#1}}%
}

\usepackage{color,lscape,longtable}

\def\boxit#1{\vbox{\hrule\hbox{\vrule\kern6pt\vbox{\kern6pt#1\kern6pt}\kern6pt\vrule}\hrule}}

\usepackage{xcolor}
\newcount\Comments  
\newcommand{\kibitz}[2]{\ifnum\Comments=1\textcolor{#1}{#2}\fi}

\usepackage{algorithm}
\usepackage{algorithmic}

\usepackage{natbib}
\usepackage{url}
\usepackage{braket}

\begin{document}

\title{Off-Policy Evaluation of Bandit Algorithm from Dependent Samples under Batch Update Policy}

\newcommand*{\affaddr}[1]{#1} 
\newcommand*{\affmark}[1][*]{\textsuperscript{#1}}
\newcommand*{\equalcontribution}[1][*]{\textsuperscript{*}}
\newcommand*{\email}[1]{\texttt{#1}}

\author{%
Masahiro Kato,\ \ \ \ \ Yusuke Kaneko\\
\affaddr{CyberAgent Inc.}\\
\email{masahiro\_kato@cyberagent.co.jp}\\
\email{kaneko\_yusuke@cyberagent.co.jp}
}

\maketitle

\begin{abstract}
The goal of \emph{off-policy evaluation} (OPE) is to evaluate a new policy using historical data obtained via a \emph{behavior policy}. However, because the contextual bandit algorithm updates the policy based on past observations, the samples are not \emph{independent and identically distributed} (i.i.d.). This paper tackles this problem by constructing an estimator from a \emph{martingale difference sequence} (MDS) for the dependent samples. In the data-generating process, we do not assume the convergence of the policy, but the policy uses the same conditional probability of choosing an action during a certain period. Then, we derive an asymptotically normal estimator of the value of an \emph{evaluation policy}. As another advantage of our method, the batch-based approach simultaneously solves deficient support problem. Using benchmark and real-world datasets, we experimentally confirm the effectiveness of the proposed method.
\end{abstract}

\section{Introduction}
As an instance of sequential decision-making problems, the \emph{multi-armed bandit} (MAB) algorithms have attracted significant attention in various applications, such as ad optimization, personalized medicine, search engines, and recommendation systems.  Recently, various methods for evaluating a new policy using historical data obtained via the MAB algorithms \citep{kdd2009_ads,li2010contextual} have emerged. The goal of \emph{off-policy evaluation} (OPE) is to evaluate a new policy by estimating the expected reward obtained from the new policy \citep{dudik2011doubly,wang2017optimal,narita2019counterfactual,pmlr-v97-bibaut19a,Kallus2019IntrinsicallyES,Oberst2019}. Although an OPE algorithm estimates the expected reward from a new policy, most existing studies presume that the samples are \emph{independent and identically distributed} (i.i.d.). However, the MAB algorithm policy updates the probability of choosing an action based on past observations, and samples are not i.i.d. owing to this update. In this case, such existing studies do not guarantee that their estimators have asymptotic normality and $\sqrt{T}$-consistency for a sample size $T$. Therefore, there is a strong motivation to establish a novel method for OPE from dependent samples. 

Several pioneering studies address OPE from dependent samples \citep{Laan2008TheCA,Hahn2011,Laan2016onlinetml,Luedtke2016,hadad2019,Kato2020}. We can group the methods for deriving asymptotic normality into the following three groups: (a) \citet{Laan2008TheCA}, \citet{Laan2016onlinetml}, \citet{hadad2019}, and \citet{Kato2020} derive the asymptotic normality with the central limit theorem (CLT) of a martingale difference sequence (MDS) by assuming that the probability of choosing an action converges to a time-invariant probability; (b) \citet{Luedtke2016} derives the asymptotic normality by standardizing a MDS, which is also used in the first group; (c) \citet{Hahn2011} apply asymptotic theory for the batched probability update process. 

This paper focuses on an approach of the third group; that is, there are sufficiently large sample sizes in each batch. Compared with \citet{Hahn2011}, our proposed method is more general and applicable in practical applications. Our method has the following three advantages compared with existing studies: (i) it does not assume convergence of the probability of choosing an action; (ii) it allows the probability of choosing an action to be $0$ for some actions in some batches; (iii) we can also use non-Donsker nuisance estimators as well as \cite{Laan2016onlinetml}. 

This paper has three main contributions. First, we provide a solution for OPE from dependent samples obtained via the MAB algorithms. Second, under the batch update policy, the proposed estimator achieves the asymptotic normality with fewer assumptions. Third, the estimator also experimentally shows a lower mean squared error (MSE) in some cases. 

\section{Problem Setting}
Here, we formulate OPE under a batch update. 

\subsection{Date-Generating Process}
Let $A_t$ be an action taking variable in $\mathcal{A}=\{1,2,\dots,K\}$, $X_t$ the \emph{covariate} observed by the decision maker when choosing an action , and $\mathcal{X}$ the space of covariate. Let us denote a random variable of a reward at period $t$ as $Y_t=\sum^K_{a=1}\mathbbm{1}[A_t = a]Y_t(a)$, where $Y_t:\mathcal{A}\to\mathbb{R}$ is a potential outcome\footnote{We can express the DGP without using the potential outcome variable \citep{kato_uehara_2020}.}. In this paper, we have access to a dataset $\{(X_t, A_t, Y_t)\}^{T}_{t=1}$ with the following data-generating process (DGP):
\begin{align}
\label{eq:DGP}
\big\{(X_t, A_t, Y_t)\big\}^{T}_{t=1}\sim p(x)\pi_t(a\mid x, \Omega_{t-1})p(y\mid a, x),
\end{align}
where $\Omega_{t-1}\in\mathcal{M}_{t-1}$ denotes the history until $t-1$ period defined as $\Omega_{t-1}=\{X_{t-1}, A_{t-1}, Y_{t-1}, \dots, X_{1}, A_1, Y_{1}\}$ with the space $\mathcal{M}_{t-1}$, $p(x)$ denotes the density of the covariate $X_t$, $\pi_t(a\mid x, \Omega_{t-1})$ denotes the probability of choosing an action $A_t$ conditioned on $X_t$ and $\Omega_{t-1}$, and $p(y\mid a, x)$ denotes the density of an outcome $Y_t$ conditioned on $A_t$ and $X_t$. We assume that $p(x)$ and $p(y\mid a, x)$ are invariant across periods, but $\pi_t(a\mid x, \Omega_{t-1})$ can take different values across periods. Let us call a policy inducing $\pi_t(a\mid x, \Omega_{t-1})$ a \emph{behavior policy}. 

\subsection{Off-Policy Evaluation}
\label{sec:opeopl}
This paper considers estimating the value of an \emph{evaluation policy} using samples obtained under the behavior policy. Let an evaluation policy $\epol:\mathcal{A}\times\mathcal{X}\to[0,1]$ be a probability of choosing an action $A_t$ conditioned on a covariate $X_t$. We are interested in estimating the expected reward from any pre-specified evaluation policy $\pi^{\mathrm{e}}(a \mid x)$. Then, we define the expected reward under an evaluation policy as $R(\epol) := \mathbb{E}\left[\sum^K_{a=1}\epol(a \mid x)Y_t(a)\right]$. For brevity, we also denote $R(\epol)$ as $\theta_0$. The goal of this paper is to estimate $\theta_0=R(\pi^{\mathrm{e}})$ using dependent samples under a batch update policy. To identify $\theta_0$, we assume overlaps in policy and the boundedness of the outcome. 
\begin{assumption}\label{asm:overlap_pol}
There exists a constant $C_1$ such that $0\leq \frac{\epol(a\mid x)}{\pi_t(a\mid x, \Omega_{t-1})}\leq C_1$.
\end{assumption}
\begin{assumption}\label{asm:overlap_outcome}
There exists a constant $C_2$ such that $|Y| \leq C_2$.
\end{assumption}

\begin{remark}[Existing Methods for OPE]
We review three types of standard estimators of $R(\pi^{\mathrm{e}})$ under the case where $\pi_1(a\mid x, \Omega_{0})=\pi_2(a\mid x, \Omega_{1})=\cdots=\pi_T(a\mid x, \Omega_{T-1})=p(a\mid x)$ in the DGP defined in \eqref{eq:DGP}. The first estimator is an inverse probability weighting (IPW) estimator given by $\frac{1}{T}\sum^T_{t=1}\sum^K_{a=1}\frac{\epol(a\mid X_t)\mathbbm{1}[A_t=a]Y_t}{p(a\mid X_t)}$ \citep{rubin87,hirano2003efficient,swaminathan15a}. Although this estimator is unbiased when the behavior policy is known, it suffers from high variance. The second estimator is a direct method (DM) estimator $\frac{1}{T}\sum^T_{t=1}\sum^K_{a=1}\hat{f}_{T}(a, X_t)$, where $\hat{f}_{T}(a, X_t)$ is an estimator of $f^*(a, X_t)$ \citep{HahnJinyong1998OtRo}. This estimator is known to be weak against model misspecification for $f^*(a, X_t)$. The third estimator is an augmented IPW (AIPW) defined as $\frac{1}{T}\sum^T_{t=1}\sum^K_{a=1}\Bigg(\frac{\epol(a\mid X_t)\mathbbm{1}[A_t=a]\big(Y_t - \hat{f}_{T}(a, X_t)\big)}{p(a\mid X_t)} + \epol(a\mid X_t)\hat{f}_{T}(a, X_t)\Bigg)$ \citep{robins94,ChernozhukovVictor2018Dmlf}.
Under certain conditions, it is known that this estimator achieves the efficiency bound (a.k.a semiparametric lower bound), which is the lower bound of the asymptotic MSE of OPE among regular $\sqrt{T}$-consistent estimators \citep{VaartA.W.vander1998As}.
\end{remark}

\begin{remark}[Semiparametric Lower Bound]
\label{rem:semi_low}
The lower bound of the variance is defined for an estimator of $\theta_0$ under some posited models of the DGP. If this posited model is a parametric model, it is equal to the \Cramer-Rao lower bound. When this posited model is a semiparametric model, we can define a corresponding \Cramer-Rao lower bound \citet{bickel98}. \citet{narita2019counterfactual} gives the semiparametric lower bound of the DGP (\ref{eq:DGP}) under $\pi_1(a\mid x, \Omega_{0})=\cdots=\pi_T(a\mid x, \Omega_{T-1})=p(a\mid x)$ as $\mathbb{E}\Big[\sum^{K}_{a=1}\frac{\big(\epol(a\mid X)\big)^2v^*(a, X_t)}{p(a\mid X)}+\left(\sum^{K}_{a=1}\epol(a\mid X)f^*(a, X_t) - \theta_0\right)^2\Big]$.
\end{remark}

\paragraph{Notations:} Let us denote $\mathbb{E}[Y_t(a)\mid x]$ and $\mathrm{Var}(Y_t(a)\mid x)$ as $f^*(a, x)$ and $v^*(a, x)$, respectively. Let $\mathcal{F}$ be the class of $f^*(a, x)$. Let $\hat{f}_{t}(a, x\mid \Omega_{t-1})$ be an estimator of $f^*(a, x)$ constructed from $\Omega_{t-1}$, respectively. Let $\mathcal{N}(\mu, \mathrm{var})$ be the normal distribution with the mean $\mu$ and the variance $\mathrm{var}$. For a random variable $Z$ and function $\mu$, let $\|\mu(Z)\|_2=\int |\mu(z)|^2 p(z) dz $ be the $L^{2}$-norm.

\subsection{Patterns of Probability Update}
In this paper, based on the update of $\pi_t(a\mid x, \Omega_{t-1})$, we classify the policies into two patterns, \emph{sequential update policy} and \emph{batch update policy}. For the sequential update policy, the policy updates $\pi_t(a\mid x, \Omega_{t-1})$ at each period \citep{Laan2008TheCA,Laan2016onlinetml,Kato2020}. Under the batch update policy, after the policy continues using a fixed probability $\pi_t(a\mid x, \Omega_{t-1})$ for some periods without updates, the policy updates $\pi_t(a\mid x, \Omega_{t-1})$ \citep{Hahn2011,narita2019counterfactual}. Although the sequential update is standard in the MAB problem, we often apply batch updates in industrial applications such as ad-optimization \citep{narita2019counterfactual}. For OPE under the sequential update, \citet{Laan2008TheCA}, \citet{Laan2016onlinetml}, \citet{Luedtke2016}, \citet{hadad2019}, and \citet{Kato2020} proposed estimators with the asymptotic normality. For instance, an adaptive AIPW (A2IPW) estimator \citep{Laan2016onlinetml,hadad2019,Kato2020} has asymptotic normality if the behavior policy converges. On the other hand, we consider OPE under a batch update. Let $M$ denote the number of updates and $\tau\in I = \{1,2,\dots,M\}$ denotes the batch index. For $\tau\in I$, the probability is updated at a period $t_\tau$, where $t_\tau - t_{\tau-1} = Tr_\tau$, using samples $\{(X_t, Y_t, A_t)\}^{t_{\tau}}_{t=t_{\tau-1}}$, where $r_1 + r_2 + \cdots + r_M = 1$ and $t_0 = 0$. Thus, in addition to the DGP \eqref{eq:DGP}, we assume 
\begin{align*}
\big\{(X_t, A_t, Y_t)\big\}^{t_\tau}_{t=t_{\tau-1}}\iid p(x)\pi_\tau(a\mid x, \Omega_{t_{\tau-1}})p(y\mid a, x),
\end{align*}
where $\pi_\tau(a\mid x, \Omega_{t_{\tau-1}})$ denotes the assignment probability updated based on samples until the period $t_{\tau-1}$. 

\subsection{Related Work}
For the sequential update, \citet{Laan2008TheCA}, \citet{Laan2016onlinetml}, \citet{hadad2019}, and \citet{Kato2020} assume that the probability of choosing an action converges to a time-invariant function almost certainly; that is, $\pi_t(a\mid x, \Omega_{t-1})\xrightarrow{\mathrm{p}} \alpha(a\mid x)$, where $\alpha:\mathcal{X}\to (0,1)$. This assumption enables us to apply the CLT for MDS. \citet{Laan2016onlinetml} proposed constructing step-wise nuisance estimators, which enables us to derive asymptotic normality without Donsker's conditions of nuisance estimators. This technique is a generalization of sample-splitting, which is also called cross-fitting in a context of double/debiased machine learning \citep{klaassen1987,ZhengWenjing2011CTME,ChernozhukovVictor2018Dmlf}. For the A2IPW estimator, \citet{hadad2019} proposed using an adaptive weight for stabilizing the behavior, and \citet{Kato2020} derived concentration inequality based on the law of iterated logarithms. On the other hand, we also construct an MDS and apply the CLT, but do not assume $\pi_t(a\mid x, \Omega_{t-1})\xrightarrow{\mathrm{p}} \alpha(a\mid x)$ by using batch update policy. Instead, we assume a sufficient sample size for each batch. 

For such a non-stationary setting, \citet{Luedtke2016} also proposed an estimator with asymptotically normality for sequential update policy without using batch update policy. For deriving the asymptotic normality, \citet{Luedtke2016} used standardization for an MDS. Although the method enables us to construct an asymptotically normal estimator for various estimators, the proposed estimator only has $\sqrt{T - \ell}$-consistency for another sample size $\ell > 0$, not $\sqrt{T}$, to estimate the variances of the MDS.

As other related work, in the MAB problem, \citet{perchet2016} considered the setting of batch policy updates. In OPE, \citet{narita2019counterfactual} also discuss a similar problem setting, but they assume that samples are i.i.d. Independently, \citet{kelly2020} provided a method for deriving a confidence interval of an ordinary least squares estimator, which is a different parameter of what we want to estimate. 

\section{OPE under Batch Update Policy}
This section introduces a concept for conducting OPE under the batch update policy and a method based on the concept with its theoretical properties.

\subsection{Strategy for OPE}
For OPE under batch update policy, we consider asymptotic properties based on the assumption of $t_\tau - t_{\tau-1} \to \infty$ as $T\to \infty$ for fixed $\tau$. Because $\big\{(X_t, A_t, Y_t)\big\}^{t_\tau}_{t=t_{\tau-1}}$ is i.i.d., we can use the standard limit theorems for the partial sum of the samples to obtain an asymptotically normal estimator of $\theta_0 = R(\epol)$. However, we also have the motivation to use all samples together to increase the efficiency of the estimator. Therefore, based on the idea of \emph{generalized method of moments} (GMM), we propose an estimator of $\theta_0$ considering the sample averages of each block as an empirical moment conditions. The main difference from the standard GMM is the assumption that the samples are not i.i.d. However, for the case under the batch update, we can apply the central limit theorems (CLT) for the martingale difference sequences (MDS) by appropriately constructing an estimator. We describe the proposed method as follows. 

\subsection{Estimator for OPE}
We propose an estimator of $\theta_0$ based on a idea of GMM. For an index of batch $\tau\in I$, a function $f\in\mathcal{F}$ such that $f:\mathcal{A}\times \mathcal{X} \to \mathbb{R}$ and an evaluation policy $\epol \in \Pi$, we define $h^{\mathrm{OPE}}_t:\mathcal{X}\times\mathcal{A}\times\mathbb{R}\times I \times \Theta\times\mathcal{F}\times\Pi \to \mathbb{R}$ as $h^{\mathrm{OPE}}_t(x, k, y; \tau, \theta, f, \epol) = \frac{1}{r_\tau}\eta_t(x, k, y; \tau, \theta, f, \epol)\mathbbm{1}\big[t_{\tau - 1} < t \leq t_{\tau} \big]$, where $\eta_t(x, k, y; \tau, \theta, f, \epol) := \phi_t(x, k, y; \tau, f, \epol) - \theta$ and $\phi_t(x, k, y; \tau, f, \epol):=$
\begin{align*}
&\sum^{K}_{a=1}\epol(a\mid x)\Bigg\{\frac{\mathbbm{1}[k=a]\big\{y - f(a, x)\big\}}{\pi_{\tau}(a\mid x, \Omega_{t_{\tau - 1}})}+f(a, x)\Bigg\}.
\end{align*}
Let us note that, for $\tau\in I$, $\theta_0\in \Theta$, $f_{t-1}\in \mathcal{F}$, and $\epol\in\Pi$, the sequence $\big\{h^{\mathrm{OPE}}_t(X_t, A_t, Y_t; \tau, \theta_0, \hat{f}_{t-1}, \epol)\big\}^{T}_{t=1}$ is an MDS: for $h^{\mathrm{OPE}}_t(X_t, A_t, Y_t; \tau, \theta_0, \hat{f}_{t-1}, \epol)$, by  $\mathbb{E}[\mathbbm{1}[A_t=a]\mid X_t, \Omega_{t-1}]=\pi_{\tau}(a\mid X_t, \Omega_{t_{\tau - 1}})$, we have
\begin{align*}
&\mathbb{E}\left[h^{\mathrm{OPE}}_t(X_t, A_t, Y_t; \tau, \theta_0, \hat{f}_{t-1}, \epol)\mid \Omega_{t-1}\right]\\
& = \mathbb{E}\left[\frac{\mathbbm{1}\big[t_{\tau - 1} < t \leq t_{\tau} \big]}{r_\tau}\eta_t(x, k, y; \tau, \theta, \hat{f}_{t-1}, \epol)\mid \Omega_{t-1}\right]\\
& = \frac{\mathbbm{1}\big[t_{\tau - 1} < t \leq t_{\tau} \big]}{r_\tau}\mathbb{E}\left[\eta_t(x, k, y; \tau, \theta, \hat{f}_{t-1}, \epol)\mid \Omega_{t-1}\right]\\
& = \frac{\mathbbm{1}\big[t_{\tau - 1} < t \leq t_{\tau} \big]}{r_\tau}\times 0 = 0.
\end{align*}
Let us also define $\bm{h}^{\mathrm{OPE}}_t\left(X_t, A_t, Y_t; \theta, \hat{f}_{t-1}, \epol\right):=$
\begin{align*}
\begin{pmatrix} 
h^{\mathrm{OPE}}_t(X_t, A_t, Y_t; 1, \theta, \hat{f}_{t-1}, \epol) \\
h^{\mathrm{OPE}}_t(X_t, A_t, Y_t; 2, \theta, \hat{f}_{t-1}, \epol) \\
 \vdots \\ 
h^{\mathrm{OPE}}_t(X_t, A_t, Y_t; M, \theta, \hat{f}_{t-1}, \epol)
\end{pmatrix}.
\end{align*}
Then, the sequence $\left\{\bm{h}^{\mathrm{OPE}}_t\left(X_t, A_t, Y_t; \theta_0, \hat{f}_{t-1}, \epol\right)\right\}^{T}_{t=1}$ is an MDS with respect to $\big\{\Omega_t\big\}^{T-1}_{t=0}$; that is, $\mathbb{E}\left[\bm{h}^{\mathrm{OPE}}_t\left(X_t, A_t, Y_t; \theta_0, \hat{f}_{t-1}, \epol\right)\mid \Omega_{t-1}\right] = \bm{0}$. Using the sequence $\left\{\bm{h}^{\mathrm{OPE}}_t\left(X_t, A_t, Y_t; \theta, \hat{f}_{t-1}, \epol\right)\right\}^{T}_{t=1}$, we define an estimator of OPE as $\widehat{R}^{\mathrm{BA2IPW}}_T(\epol) := $
\begin{align}
\label{gmmdm_ope}
\argmin_{\theta\in\Theta} \left(\hat{\bm{q}}^{\mathrm{OPE}}_T(\theta)\right)^\top \hat{W}_T \left(\hat{\bm{q}}^{\mathrm{OPE}}_T(\theta)\right),
\end{align}
where $\hat{\bm{q}}^{\mathrm{OPE}}_T(\theta) = \frac{1}{T}\sum^{T}_{t=1}\bm{h}^{\mathrm{OPE}}_t\left(X_t, A_t, Y_t; \theta, \hat{f}_{t-1}, \epol\right)$ and $\hat{W}_T$ is a data-dependent $(M\times M)$-dimensional positive semi-definite matrix. Let us note that the estimator defined in Eq.~\eqref{gmmdm_ope} is an application of GMM with the moment condition $\bm{q}^{\mathrm{OPE}}(\theta_0) = \mathbb{E}\left[\frac{1}{T}\sum^T_{t=1}\bm{h}^{\mathrm{OPE}}_t\left(X_t, A_t, Y_t; \theta_0, \hat{f}_{t-1}, \epol\right)\right]=0$. For the minimization problem defined in Eq.~\eqref{gmmdm_ope}, we can analytically calculate the minimizer as $\widehat{R}^{\mathrm{BA2IPW}}_T(\epol) = w^\top_T D_T(\epol)$, where $w_T = (w_{T, 1}\ \cdots\ w_{T, M})^\top$ is an $M$-dimensional vector such that $\sum^M_{\tau=1}w_{T, \tau}=1$, and $D_T(\epol)$ is
\begin{align*}
\begin{pmatrix}
\frac{1}{t_{1}}\sum^{t_{1}}_{t=1} \phi_t(X_t, A_t, Y_t; 1, \hat{f}_{t-1}, \epol) \\
\frac{1}{t_{2} - {t_{1}}}\sum^{t_2}_{t=t_1+1} \phi_t(X_t, A_t, Y_t; 2, \hat{f}_{t-1}, \epol)\\
\vdots\\
\frac{1}{T-t_{M-1}}\sum^{T}_{t=t_{M-1}+1} \phi_t(X_t, A_t, Y_t; M, \hat{f}_{t-1}, \epol)\\
\end{pmatrix}.
\end{align*}
We call the estimator a \emph{Batch-based Adaptive AIPW} (BA2IPW) estimator. In Appendix~\ref{appdx:sec:gmmdm}, we discuss the GMM perspective in more detail. 

\subsection{Asymptotic Properties}
\label{sec:asymp_ba2ipw}
Here, we show the consistency and asymptotic normality of the proposed BA2IPW estimator $\widehat{R}^{\mathrm{BA2IPW}}_T(\epol)$. 

\begin{theorem}[Consistency of the BA2IPW Estimator]
\label{thm:consistency}
Suppose that there exists a constant $C_f > 0$ such that $\big|f_{t_{\tau-1}}(a, x)\big| < C_f$ for $\tau \in I$. Then, under Assumptions~\ref{asm:overlap_pol} and \ref{asm:overlap_outcome}, $\widehat{R}^{\mathrm{BA2IPW}}_T(\epol)\xrightarrow{\mathrm{p}}\theta_0$. 
\end{theorem}
\begin{proof}
We use the law of large numbers for an MDS from the boundedness of $\bm{h}^{\mathrm{OPE}}_t\left(X_t, A_t, Y_t; \theta_0, \hat{f}_{t-1}, \epol\right)$, and we have $\frac{1}{T}\sum^T_{t=1}\bm{h}^{\mathrm{OPE}}_t\left(X_t, A_t, Y_t; \theta_0, \hat{f}_{t-1}, \epol\right)\xrightarrow{\mathrm{p}}0$ (Proposition~\ref{prp:mrtgl_WLLN} in Appendix~\ref{sec:prelim}). This result means that $D_T(\epol)\xrightarrow{\mathrm{p}}I\theta_0$, where $I = (1\ 1\ \cdots\ 1)^\top$ is an $M$-dimensional vector. Therefore, $\widehat{R}^{\mathrm{BA2IPW}}_T(\epol) = w_T D_T(\epol) \xrightarrow{\mathrm{p}} w_T I \theta_0 = \theta_0$.
\end{proof}
\begin{theorem}[Asymptotic Distribution of the BA2IPW Estimator]
Suppose that (i) $w_T = (w_{T,1}\ \cdots\ w_{T,M})^\top \xrightarrow{\mathrm{p}} w = (w_{1}\ \cdots\ w_{M})^\top$; (ii) $w_{T,\tau} > 0$ and $\sum^M_{\tau=1}w_{T,\tau} = 1$; (iii) $\hat{f}_{t-1}(a, x) \xrightarrow{\mathrm{p}} f^*(a, x)$ for all $a \in \mathcal{A}$ and $x \in \mathcal{X}$; (iv) There exists a constant $C_f > 0$ such that $\big|\hat{f}_{t-1}(a, x)\big| < C_f$. Then, under Assumptions~\ref{asm:overlap_pol} and \ref{asm:overlap_outcome}, $\sqrt{T}\big(\widehat{R}^{\mathrm{BA2IPW}}_T(\epol) - \theta_0\big)  \xrightarrow{\mathrm{d}}\mathcal{N}\big(0, \sigma^2\big)$, where $\sigma^2 = \sum^M_{\tau=1}w_{\tau}\sigma^2_{\tau}$ and $\sigma^2_{\tau}=\frac{1}{r_\tau}\mathbb{E}\Big[\sum^{K}_{a=1}\frac{\big(\epol(a\mid X)\big)^2\nu^*(a, X)}{\pi_{\tau}(a\mid X, \Omega_{t_{\tau - 1}})} + \left(\sum^{K}_{a=1}\epol(a\mid X)f^*(a, X) - \theta_0\right)^2\Big]$.
\end{theorem}
The proof is shown in Appendix~\ref{appdx:main}. Readers might consider that the use of MDS for deriving the asymptotic normality is unnecessary. We discuss the necessity of MDS in Appendix~\ref{appdx:nec_mds}. We can also define a corresponding \emph{Batch-based Adaptive IPW} (BAdaIPW). For the BAdaIPW estimator, the variance of a batch $\tau$ is  $\sigma^2_{\mathrm{IPW}, \tau}=\frac{1}{r_\tau}\mathbb{E}\left[\sum^{K}_{a=1}\frac{\big(\epol(a\mid X)\big)^2\mathbb{E}[Y^2_t\mid X_t]}{\pi_{\tau}(a\mid X, \Omega_{t_{\tau - 1}})} - \theta^2_0\right]$. 

\begin{remark}[Construction of $f_{t_m}$ and Donsker Condition]
As well as the cross-fitting of double/debiased machine learning proposed by \citet{klaassen1987},\citet{ZhengWenjing2011CTME},\citet{Laan2016onlinetml}, and \citet{ChernozhukovVictor2018Dmlf}, the proposed estimator does not require Donsker's condition for asymptotic normality. This property comes from the MDS as pointed by \citet{Laan2016onlinetml}. On the other hand, because the samples are not independent, we cannot use the standard regression to obtain a consistent estimator of $f^*$. For example, \citet{yang2002} propose a nonparametric method for the bandit process under some mild conditions. 
\end{remark}

\subsection{Weight of the Proposed Estimator}
\label{sec:efficiency}
Next, we discuss the choice of weight $w_T$.

\paragraph{Equal Weight:} A naive choice is weighting the moment conditions equality; that is, $w_{T, \tau} = \frac{1}{M}$. In this case, the proposed estimator boils down to $I^\top D_T(\epol)$, which is almost the same as the A2IPW estimator. Although the estimator itself is similar to the A2IPW estimator, the theoretical guarantee for the asymptotic normality is different. While the A2IPW estimator uses the assumption that the policy converges to a time-invariant policy, the proposed BA2IPW estimator uses the assumption of the batch update. We call the BA2IPW estimator with the equal weight a Plain BA2IPW (PBA2IPW) estimator.

\paragraph{Efficient Weight:} First, we consider an efficient weight $w_T$ that minimizes the asymptotic variance of $\widehat{R}^{\mathrm{BA2IPW}}_T(\epol)$. As well as the standard GMM, the $\tau$-th element of the efficient weight is given as $w^*_{\tau}=\frac{1}{\sigma^2_\tau}/\sum^M_{\tau'=1}\frac{1}{\sigma^2_{\tau'}}$ \citep{GVK126800421}. Here, we use the orthogonality among moment conditions; that is, zero covariance. In this case, the asymptotic variance becomes $1/\sum^M_{\tau'=1}\frac{1}{\sigma^2_{\tau'}}$. Therefore, for gaining efficiency, we use a weight $\hat{w}_{T,\tau} = \frac{1}{ \hat{\sigma}^2_{T, \tau}}/\sum^M_{\tau'=1}\frac{1}{\hat{\sigma}^2_{T, \tau'}}$, where $\hat{\sigma}^2_{T, \tau}$ is an estimator of $\sigma^2_\tau$. We call the BA2IPW estimator with the efficient weight an Efficient BA2IPW (EBA2IPW) estimator. 

\begin{remark}[Estimation of $w^*$]
If $\hat{\sigma}^2_{T, \tau}\xrightarrow{\mathrm{p}}\sigma^2_\tau$, we also have $\hat{w}_{T,\tau} \xrightarrow{\mathrm{p}}w^*_\tau$ from the continuous mapping theorem. In this paper, we propose two estimators defined as $\hat{\sigma}^2_{T, \tau}=\frac{1}{r_\tau T}\sum^{T}_{t=1}\left\{h^{\mathrm{OPE}}_t(x, a, y; \tau, \widehat{R}^{\mathrm{BA2IPW}}_T(\epol), \hat{f}_{t-1}, \epol)\right\}^2\times$ $\mathbbm{1}\big[t_{\tau - 1} < t \leq t_{\tau} \big]$ and $\tilde{\sigma}^2_{T, \tau}=\frac{1}{r_\tau T}\sum^{T}_{t=1}\left\{h^{\mathrm{OPE}}_t(x, a, y; \tau, \widehat{R}^{\mathrm{BA2IPW}}_T(\epol), \hat{f}_{T}, \epol)\right\}^2\times$ $\mathbbm{1}\big[t_{\tau - 1} < t \leq t_{\tau} \big]$.
\end{remark}

\paragraph{Weight for Numerical Stability:} 
As explained above, we can obtain an efficient weight that minimizes the asymptotic variance. However, as \citet{hadad2019} pointed out, when using time-variance estimators $f_t$, early variable $\phi_t(X_t, A_t, Y_t; \tau, f_t, \epol)$ might be unstable due to the existence of an inaccurate estimators $f_t$ at early stages. \citet{hadad2019} proposed an adaptive weight that puts more weight on later variables $\phi_t(X_t, A_t, Y_t; \tau, f_t, \epol)$ in the estimator. In our problem setting, we can also introduce such weights. However, unlike the adaptive weight proposed by \citet{hadad2019}, which must be martingale, we do not require martingales on the weights. This property is a benefit of the batch update. For instance, for stabilization, we define a weight $\frac{1}{\ddot{\sigma}^2_{T, \tau}}/\sum^M_{\tau'=1}\frac{1}{\ddot{\sigma}^2_{T, \tau'}}$, where $\ddot{\sigma}^2_{T, \tau} = \hat{\sigma}^2_{T, \tau}+ \alpha\frac{1}{t_\tau}\sum^{t_{\tau}}_{t=t_{\tau-1}+1} \big(f_t(A_t, X_t) - f_T(A_t, X_t)\big)^2$ and $\alpha > 0$ is a constant. The first term is an efficient weight described above. The second term reflects the deviation between $f_t$ and $f_T$, which would be more accurate because it uses more samples. 

\subsection{Main Algorithm}
\label{sec:main_algorithm}
As discussed in Section~\ref{sec:efficiency}, we can minimize the asymptotic variance of the proposed estimator $\widehat{R}^{\mathrm{BA2IPW}}_T(\epol)$ by choosing $w_T$ appropriately. However, to obtain $w_{T,\tau}\xrightarrow{\mathrm{p}} \frac{1}{\sigma^2_\tau}/\sum^M_{\tau'=1}\frac{1}{\sigma^2_{\tau'}}$, which is the optimal weight that minimizes the asymptotic variance, we need a consistent estimator of $\theta_0$, which is what we want to estimate. On the other hand, we have a consistent estimator of $\theta_0$ without using an optimal weight matrix $w_{T,\tau}\xrightarrow{\mathrm{p}} \frac{1}{\sigma^2_\tau}/\sum^M_{\tau'=1}\frac{1}{\sigma^2_{\tau'}}$. Based on these properties, we propose \emph{two-step estimation}. First, using an arbitrary positive definite weight $w^{(0)}_T$, such as the identity matrix, we obtain an initial estimate $\widehat{R}^{\mathrm{BA2IPW}, (1)}_T(\epol)$. Then, using $\widehat{R}^{\mathrm{BA2IPW}, (1)}_T(\epol)$, we construct $w^{(1)}_{T,\tau}\xrightarrow{\mathrm{p}} \frac{1}{\sigma^2_\tau}/\sum^M_{\tau'=1}\frac{1}{\sigma^2_{\tau'}}$. We can obtain an efficient estimator $\widehat{R}^{\mathrm{BA2IPW}, (2)}_T(\epol)$ of $\theta_0$, as discussed. More generally, we consider an algorithm with iteration such that after obtaining $\widehat{R}^{\mathrm{BA2IPW}, (i-1)}_T(\epol)$, we estimate $w^{(i-1)}_{T,\tau}\xrightarrow{\mathrm{p}} \frac{1}{\sigma^2_\tau}/\sum^M_{\tau'=1}\frac{1}{\sigma^2_{\tau'}}$ and obtain a next estimator $\widehat{R}^{\mathrm{BA2IPW}, (i)}_T(\epol)$ by using $w^{(i-1)}_{T,\tau}$. We refer to this algorithm with $N$-iterations as \emph{$N$-step BA2IPW estimation}. We can use sufficiently large $N$ because, at each iteration, we only calculate the weighted average of the moment conditions using $\widehat{R}^{\mathrm{BA2IPW}, (i)}_T(\epol)$, which is not time-consuming. Although the asymptotic properties of the iterated estimator are the same as those of the two-step estimator, we report that the iteration improves the empirical performance in some cases. For $N\geq 2$, we summarize the $N$-Step BA2IPW Estimation in Algorithm~\ref{alg}. We can use any method to construct $f_{t_\tau}$ and $\pi_\tau$ as long as they are consistent for bandit data and satisfy some regularity conditions needed for Theorem~\ref{thm:main}. 

\begin{algorithm}[tb]
   \caption{$N$-step efficient BA2IPW estimation}
   \label{alg}
\begin{algorithmic}
   \STATE {\bfseries Input:} $\big\{(X_t, A_t, Y_t)\big\}^{T}_{t=1}$ and  $\big\{f_{t}\big\}^{T-1}_{t=0}$. 
   \STATE {\bfseries Initialization:} Let $w^{(0)}_T$ be a positive definite matrix such as the identity matrix.
   \FOR{$i=1$ to $N$}
   \STATE Using $w^{(i-1)}_T$, obtain $\widehat{R}^{\mathrm{BA2IPW}, (i)}_T(\epol)  = w^{(i-1)}_T D_T(\epol)$.
   \STATE Using $\widehat{R}^{\mathrm{BA2IPW}, (i)}_T(\epol)$, construct weights $\{\hat{w}^{(i)}_T\}$. 
   \ENDFOR
   \STATE{\bfseries Output:} Estimator $\widehat{R}^{\mathrm{BA2IPW}}_T(\epol) = \widehat{R}^{\mathrm{BA2IPW}, (N)}_T(\epol)$
\end{algorithmic}
\end{algorithm}

\section{Deficient Support Problem}
As an application of BA2IPW, we consider an OPE without Assumption~\ref{asm:overlap_pol}, which assumes that there exists $C_1$ such that $0\leq \frac{\epol(a\mid x)}{p_t(a\mid x)}\leq C_1$. Instead of Assumption~\ref{asm:mds}, we consider a situation in which we are allowed to change the support of actions in each batch. For example, in the first batch, we choose an action from a set $\{1,2,3\}$ with a probability larger than $0$, but we choose an action from a set $\{1,2,4\}$ with a probability larger than $0$ in the second batch. In this case, the probability of choosing the action $4$ is $0$ in the first batch, while the probability of choosing the  action $3$ is $0$ in the second batch. This situation is a common in practice and called deficient support problem \citep{Sachdeva2020}. For this problem, instead of Assumption~\ref{asm:overlap_pol}, we use the following assumption.
\begin{assumption}\label{asm:overlap_pol2}
For $a\in\{1,2,\dots,K\}$, there exist $\tau\in\{1,2,\dots,M\}$ and $C_1$ such that $0\leq \frac{\epol(a\mid x)}{\pi_{\tau}(a\mid x, \Omega_{t_{\tau - 1}})}\leq C_1$.
\end{assumption}
Under this assumption, if $\pi_{\tau}(a\mid x, \Omega_{t_{\tau - 1}})>0$ for at least one batch, we are allowed to use $\pi_{\tau'}(a\mid x, \Omega_{t_{\tau' - 1}})=0$ for $\tau'\neq \tau$. With this assumption, we derive the asymptotic normality in Appendix~\ref{appdx:incomp_support_of_arms}. Thus, our approach provides a new solution to this problem.

\begin{table*}[t]
\begin{center}
\caption{Results of OPE under the RW policy. We highlight in bold the best two estimators in each dataset.} 
\medskip
\label{tbl:exp_table1}
\vspace{-0.3cm}
\scalebox{0.70}[0.70]{
\begin{tabular}{l|rr|rr|rr|rr|rr|rr}
\toprule
Datasets &  \multicolumn{2}{c|}{satimage}& \multicolumn{2}{c|}{pendigits}& \multicolumn{2}{c|}{mnist}& \multicolumn{2}{c|}{letter}& \multicolumn{2}{c|}{sensorless}& \multicolumn{2}{c}{connect-4} \\
Metrics &      MSE &      SD &      MSE &      SD &      MSE &      SD &      MSE &      SD &      MSE &      SD &     MSE &     SD \\
\hline
PBA2IPW &  \textbf{0.038} &  0.002 &  \textbf{0.129} &  0.047 &  \textbf{0.173} &  0.104 &  \textbf{0.331} &  0.547 &  \textbf{0.146} &  0.089 &  \textbf{0.021} &  0.021 \\
EBA2IPW &  0.050 &  0.006 &  0.190 &  0.030 &  0.191 &  0.025 &  0.398 &  0.062 &  0.182 &  0.028 &  \textbf{0.025} &  0.022 \\
EBA2IPW' &  0.044 &  0.003 &  0.182 &  0.029 &  \textbf{0.102} &  0.012 &  0.389 &  0.064 &  0.177 &  0.027 &  \textbf{0.025} &  0.024 \\
BAdaIPW &  0.077 &  0.010 &  0.178 &  0.082 &  0.200 &  0.111 &  0.333 &  0.537 &  0.160 &  0.093 &  0.027 &  0.027 \\
AdaDM &  0.141 &  0.010 &  0.493 &  0.034 &  0.434 &  0.036 &  0.476 &  0.023 &  0.413 &  0.031 &  0.142 &  0.024 \\
AIPW &  \textbf{0.032} &  0.001 &  \textbf{0.110} &  0.030 &  0.244 &  0.028 &  \textbf{0.254} &  0.216 &  \textbf{0.128} &  0.064 &  0.055 &  0.022 \\
DM &  0.099 &  0.004 &  0.452 &  0.025 &  0.282 &  0.028 &  0.459 &  0.023 &  0.395 &  0.025 &  0.086 &  0.018 \\
\bottomrule
\end{tabular}
} 
\end{center}
\vspace{-0.3cm}
\begin{center}
\caption{Results of OPE under the UCB policy. We highlight in bold the best two estimators in each dataset.} 
\medskip
\label{tbl:exp_table2}
\vspace{-0.cm}
\scalebox{0.70}[0.70]{
\begin{tabular}{l|rr|rr|rr|rr|rr|rr}
\toprule
Datasets &  \multicolumn{2}{c|}{satimage}& \multicolumn{2}{c|}{pendigits}& \multicolumn{2}{c|}{mnist}& \multicolumn{2}{c|}{letter}& \multicolumn{2}{c|}{sensorless}& \multicolumn{2}{c}{connect-4} \\
Metrics &      MSE &      SD &      MSE &      SD &      MSE &      SD &      MSE &      SD &      MSE &      SD &     MSE &     SD \\
\hline
PBA2IPW &  0.050 &  0.005 &  0.088 &  0.015 &  0.240 &  0.415 &  \textbf{0.205} &  0.088 &  \textbf{0.162} &  0.057 &  0.032 &  0.032 \\
EBA2IPW &  \textbf{0.014} &  0.000 &  \textbf{0.029} &  0.002 &  0.240 &  0.070 &  0.434 &  0.036 &  0.239 &  0.054 &  0.030 &  0.022 \\
EBA2IPW' &  0.036 &  0.005 &  0.081 &  0.017 &  \textbf{0.142} &  0.030 &  0.422 &  0.032 &  0.202 &  0.038 &  0.037 &  0.033 \\
BAdaIPW &  0.087 &  0.037 &  0.122 &  0.030 &  0.275 &  0.406 &  0.219 &  0.090 &  0.183 &  0.067 &  0.057 &  0.057 \\
AdaDM &  0.076 &  0.002 &  0.230 &  0.008 &  0.372 &  0.020 &  0.451 &  0.018 &  0.327 &  0.021 &  0.071 &  0.023 \\
AIPW &  0.036 &  0.003 &  0.058 &  0.012 &  \textbf{0.136} &  0.007 &  \textbf{0.170} &  0.043 &  \textbf{0.134} &  0.038 &  \textbf{0.022} &  0.012 \\
DM &  \textbf{0.009} &  0.000 &  \textbf{0.049} &  0.001 &  0.161 &  0.008 &  0.371 &  0.019 &  0.214 &  0.013 &  \textbf{0.019} &  0.010 \\
\bottomrule
\end{tabular}
} 
\end{center}
\vspace{-0.3cm}
\end{table*}

\section{Estimation of the Behavior Policy}
In the proposed BA2IPW method, we assume that the true behavior policy is known. However, in many real-world applications, the assumption does not hold. To solve this problem, by using an estimator $\hat{g}_{t-1}$ of $\pi_\tau$, which is constructed from $\Omega_{t-1}$ as well as $\hat{f}_{t-1}$, we also propose a \emph{Batch-based Adaptive Doubly Robust} (BADR) as $\hat{R}^{\mathrm{BADR}}_T(\epol)  = w^\top_T \widetilde{D}_T(\epol)$, where $w_T = (w_{T, 1}\ \cdots\ w_{T, M})^\top$ is an $M$-dimensional vector such that $\sum^M_{\tau=1}w_{T, \tau}=1$, $\tilde{D}_T(\epol)=$
\begin{align*}
\begin{pmatrix}
\frac{1}{t_{1}}\sum^{t_{1}}_{t=1} \tilde{\phi}_t(X_t, A_t, Y_t; 1, \hat{f}_{t-1}, \hat{g}_{t-1}, \epol) \\
\frac{1}{t_{2} - {t_{1}}}\sum^{t_2}_{t=t_1+1} \tilde{\phi}_t(X_t, A_t, Y_t; 2, \hat{f}_{t-1}, \hat{g}_{t-1}, \epol)\\
\vdots\\
\frac{1}{T-t_{M-1}}\sum^{T}_{t=t_{M-1}+1} \tilde{\phi}_t(X_t, A_t, Y_t; M, \hat{f}_{t-1}, \hat{g}_{t-1}, \epol)\\
\end{pmatrix},
\end{align*}
and $\tilde{\phi}_t(x, k, y; \tau, f, g, \epol):=$
\begin{align*}
&\sum^{K}_{a=1}\epol(a\mid x)\Bigg\{\frac{\mathbbm{1}[k=a]\big\{y - f(a, x)\big\}}{g(a\mid x)}+f(a, x)\Bigg\}.
\end{align*}
For the BADR estimator, we show the asymptotic normality as follows. The proof is shown in Appendix~\ref{appdx:badr_dist}. 

\begin{theorem}[Asymptotic Distribution of the BADR Estimator]
\label{thm:badr_asymp}
Suppose that (i) $w_T = (w_{T,1}\ \cdots\ w_{T,M})^\top \xrightarrow{\mathrm{p}} w = (w_{1}\ \cdots\ w_{M})^\top$; (ii) $w_{T,\tau} > 0$ and $\sum^M_{\tau=1}w_{T,\tau} = 1$; (iii) for $\alpha\beta=\op((t-t_{\tau-1})^{-1/2}),\alpha=\op(1), \beta=\op(1)$, and each $\tau\in I$, the nuisance estimators satisfy $\|\hat{g}_{t-1}(a\mid X_t) - \pi_{\tau}(a\mid X_t, \Omega_{t_{\tau-1}})\|_{2}=\alpha$, and $\|\hat{f}_{t-1}(a,X_t)-f^*(a,X_t)\|_2=\beta$, where the expectation of the norm is over $X_t$; (iv) there exit constants $C_f$ and $C_g$ such that $|\hat{f}_{t-1}(a, x)| \leq C_f$ and $0 < \left|\frac{\epol(a\mid x)}{\hat{g}_{t-1}(a\mid x)}\right| \leq C_g$ for all $a\in\mathcal{A}$ and $x\in\mathcal{X}$. Then, under Assumptions~\ref{asm:overlap_pol} and \ref{asm:overlap_outcome}, $\sqrt{T}\big(\hat{R}^{\mathrm{BADR}}_T(\epol)  - \theta_0\big)  \xrightarrow{\mathrm{d}}\mathcal{N}\big(0, \sigma^2\big)$.
\end{theorem}

\section{Off-Policy Learning}
An important application of OPE is Off-Policy Learning (OPL), which attempts to determine the optimal policy maximizing the expected reward. Let us define the optimal policy $\pi^*$ as $\pi^* = \argmax_{\pi\in \Pi} R(\pi)$, where $\Pi$ is a policy class. By applying each OPE estimator, we estimate the optimal policy as $\hat{\pi} = \argmax_{\pi\in \Pi}\widehat{R}^{\mathrm{BA2IPW}}_T(\pi)$.

\begin{table*}[t]
\begin{center}
\caption{The coverage ratios (CRs) are shown. The left graph shows the results with the RW policy. The right graph shows the results with the UCB policy.} 
\medskip
\label{tbl:exp_table3}
\vspace{-0.3cm}
\begin{tabular}{cc}

\begin{minipage}{0.5\hsize}
\begin{center}
\scalebox{0.70}[0.70]{
\begin{tabular}{l|r|r|r|r|r|r}
\toprule
Datasets &  satimage& pendigits& mnist& letter& sensorless& connect-4 \\
\hline
PBA2IPW &  1.00 &  0.96 &  1.00 &  0.85 &  0.94 &  1.00 \\
EBA2IPW &  0.66 &  0.22 &  0.15 &  0.02 &  0.19 &  0.88 \\
EBA2IPW' &  0.67 &  0.27 &  0.47 &  0.02 &  0.19 &  0.89 \\
BAdaIPW &  0.93 &  0.90 &  0.90 &  0.83 &  0.91 &  0.99 \\
\bottomrule
\end{tabular}
} 
\end{center}
\end{minipage}

\begin{minipage}{0.5\hsize}
\begin{center}
\scalebox{0.70}[0.70]{
\begin{tabular}{l|r|r|r|r|r|r}
\toprule
Datasets &  satimage& pendigits& mnist& letter& sensorless& connect-4 \\
\hline
PBA2IPW &  1.00 &  1.00 &  1.00 &  0.94 &  0.98 &  1.00 \\
EBA2IPW &  0.92 &  0.74 &  0.01 &  0.00 &  0.02 &  0.78 \\
EBA2IPW' &  0.77 &  0.58 &  0.17 &  0.00 &  0.03 &  0.67 \\
BAdaIPW &  0.91 &  0.88 &  0.91 &  0.89 &  0.82 &  0.93 \\
\bottomrule
\end{tabular}
} 
\end{center}
\end{minipage}
\end{tabular}
\end{center}
\vspace{-0.4cm}
\end{table*}

\begin{table*}[t]
\begin{center}
\caption{Results of OPL under the RW policy. We highlight in bold the best two estimators in each dataset.} 
\medskip
\label{tbl:exp_table4}
\vspace{-0.3cm}
\scalebox{0.73}[0.73]{
\begin{tabular}{l|rr|rr|rr|rr|rr|rr}
\toprule
Datasets &  \multicolumn{2}{c|}{satimage}& \multicolumn{2}{c|}{pendigits}& \multicolumn{2}{c|}{mnist}& \multicolumn{2}{c|}{letter}& \multicolumn{2}{c|}{sensorless}& \multicolumn{2}{c}{connect-4} \\
Metrics &      RWD &      SD &      RWD &      SD &      RWD &      SD &      RWD &      SD &      RWD &      SD &     RWD &     SD \\
\hline
PBA2IPW &  0.812 &  0.020 &  0.690 &  0.072 &  0.493 &  0.298 &  \textbf{0.172} &  0.090 &  0.264 &  0.131 &  0.665 &  0.030 \\
EBA2IPW &  0.813 &  0.022 &  \textbf{0.717} &  0.063 &  \textbf{0.519} &  0.313 &  0.135 &  0.073 &  \textbf{0.289} &  0.126 &  \textbf{0.679} &  0.024 \\
BAdaIPW &  \textbf{0.815} &  0.023 &  \textbf{0.697} &  0.089 &  \textbf{0.515} &  0.312 &  0.150 &  0.078 &  \textbf{0.313} &  0.145 &  0.677 &  0.024 \\
AdaDM &  0.777 &  0.033 &  0.478 &  0.076 &  0.191 &  0.151 &  0.046 &  0.026 &  0.190 &  0.091 &  0.654 &  0.023 \\
AIPW &  \textbf{0.819} &  0.020 &  0.698 &  0.062 &  \textbf{0.515} &  0.312 &  \textbf{0.154} &  0.087 &  0.287 &  0.141 &  \textbf{0.678} &  0.024 \\
DM &  0.791 &  0.034 &  0.544 &  0.071 &  0.247 &  0.174 &  0.057 &  0.036 &  0.210 &  0.090 &  0.654 &  0.023 \\
\bottomrule
\end{tabular}
} 
\end{center}
\vspace{-0.5cm}
\end{table*}

\section{Experiments}
Using benchmark datasets and real-world logged data, we demonstrate the effectiveness of the BA2IPW estimator with an equal weight (PBA2IPW) and efficient weight using variance estimators $\hat{\sigma}^2_{T,\tau}$ (EBA2IPW) and $\tilde{\sigma}^2_{T,\tau}$ (EBA2IPW'), and BAdaIPW estimator with an equal weight (BAdaIPW). Note that although the forms of the several estimators are the same as the existing studies, the theoretical guarantees are different.

\subsection{Experiments with Benchmark Dataset} 
Following \citet{dudik2011doubly} and \citet{Chow2018}, we evaluate the proposed estimators using classification datasets by transforming them into contextual bandit data. From the LIBSVM repository, we use the satimage, pendigits, mnist, letter, sensorless, and connect-4 datasets \footnote{\url{https://www.csie.ntu.edu.tw/~cjlin/libsvmtools/datasets/}}. For a batched update behavior policy, we use the random walk (RW) and LinearUCB (UCB) \citep{sutton1998reinforcement,li2010contextual,Wei2011} policies. When using the RW policy, we first decide the probability of choosing an action from uniform distribution. Then, we add a noise $0.01\mathcal{N}(0,1)$ at each batch, i.e., the policy is a random walk. At each batch, we standardize the values of random walk to be probability, i.e., all values are positive and the sum is $1$. When using the UCB policy, we choose an estimated best arm $\hat{A}_t$ firstly. Then, we create an adaptive policy. Then, we construct a behavior policy as a policy that chooses $\hat{A}_t$ with probability $0.8$ and the other arms with equal probability. While the probability of choosing an action of the UCB policy converges, that of the RW policy does not converge. 

For each dataset, we compare the performances of the following estimators of policy value: PBA2IPW, EBA2IPW, EBA2IPW', BAdaIPW, Adaptive DM estimator (AdaDM) defined as $\frac{1}{T}\sum^T_{t=1}\sum^K_{a=1}\epol(a\mid X_t)\hat{f}_{t}(a, X_t)$, an AIPW defined as $\frac{1}{T}\sum^T_{t=1}\sum^K_{a=1}\frac{\epol(a\mid X_t)\mathbbm{1}[A_t=a]\big(Y_t - \hat{f}_{T}(a, X_t)\big)}{\pi_t(a\mid X_t, \Omega_{t-1})} + \frac{1}{T}\sum^T_{t=1}\sum^K_{a=1}\epol(a\mid X_t)\hat{f}_{T}(a, X_t)$, and DM estimator (DM) defined as $\frac{1}{T}\sum^T_{t=1}\sum^K_{a=1}\epol(a\mid X_t)\hat{f}_{T}(a, X_t)$. When estimating $f^*$, we use the Nadaraya-Watson regression (NW) estimator \citep{yang2002}. 

\paragraph{MSEs:} To construct an evaluation policy, we create a deterministic policy $\pi^d$ by training a logistic regression classifier on historical data and set the output as $\pi^d$. Through experiments, the behavior policy $\pi_t$ is assumed to be known. More details are in Appendix~\ref{appdx:det_exp}. Let us construct the evaluation policy $\epol$ as a mixture of $\pi^d$ and the uniform random policy $\pi^u$, defined as $\epol = 0.9\pi^d+0.1\pi^u$. We construct the evaluation policy $\epol$ as a mixture of $\pi^d$ and the uniform random policy $\pi^u$ defined as $\epol = 0.9\pi^d+0.1\pi^u$. We compare the MSEs of six estimators, the PBA2IPW, EBA2IPW, EBA2IPW', BAdaIPW, AdaDM, AIPW, and DM estimators. For estimating the weights of EBA2IPW and EBA2IPW' estimators, we iteratively estimate the wights and the value $\theta_0$ $10$ times. In each experiment, we have historical data with a sample size $T=1500$. When estimating $f^*$, we use the Nadaraya-Watson regression (NW) estimator \citep{yang2002}. 

The resulting MSEs and their standard deviations (SDs) over $100$ replications of each experiment are shown in Tables~\ref{tbl:exp_table1} and \ref{tbl:exp_table2}. In many cases, the proposed methods show the preferable the existing methods. When using the RW policy, the policy does not converge to a time-invariant policy. Therefore, the proposed method is theoretically preferable for the situation. More importantly, we can construct confidence intervals from the proposed methods, but cannot construct it from the A2IPW estimator. On the other hand, when using the UCB policy, the policy converges to a time-invariant policy, but the proposed methods still show higher performance for various datasets in some cases. We consider that this result is based on the fact that, even though the policy approaches to a time-invariant policy, the update is only allowed in batch, and does not converge sufficiently. In Appendix~\ref{appdx:det_exp}, we show the additional results. 

\paragraph{Coverage Ratio of Confidence Interval:} In Table~\ref{tbl:exp_table3}, we show the coverage ratio of the confidence intervals derived in the previous experiments together with the MSEs. The coverage ratio of the confidence interval is a percentage at which it covers the true value $\theta_0$ in the confidence interval. For the $100$ trials of the previous experiment, we calculate the coverage ratio (CR) of $95\%$ confidence interval, which is constructed as $\left[\hat{\theta}_T - 1.96\sqrt{\frac{\hat{\sigma}^2}{1500}}, \hat{\theta}_T + 1.96\sqrt{\frac{\hat{\sigma}^2}{1500}}\right]$, where $\hat{\theta}_T$ is an estimator of $\theta_0$ and $\hat{\sigma}^2$ is its estimated asymptotic variance. In the results, the PBA2IPW and BAdaIPW estimator shows CR close to $0.95$ in many cases. The BAdaIPW estimator does not require an estimator of $f^*$. Therefore, compared with the other estimators, it shows more preferable performances. EBA2IPW estimator requires several variance estimators, and their estimation error worsen the results compared with the PBA2IPW and BAdaIPW estimators.

\paragraph{OPL:}
In the experiments of OPL, we compare the performances of estimated policy maximizing expected reward obtained from the PBA2IPW, EBA2IPW, BAdaIPW, AdaDM, AIPW, and DM estimators. We conducted $5$ trials for each experiment with $T=1500$. The resulting expected rewards over the evaluation data (RWDs) and the SDs are shown in Table~\ref{tbl:exp_table4}, where we highlight in bold the best two estimators. 

\subsection{Experiments with Real-World Data}
We apply our estimators to evaluate a policy using the real-world dataset in CyberAgent Inc., which is the second-largest Japanese advertisement company with about $7$ billion USD market capitalization (as of August 2020). This company simultaneously runs Thompson sampling and uniformly random sampling to determine the design of advertisements. The Thompson sampling updates the parameter every $30$ minute; therefore, there are batches with samples obtained during the $30$ minute. We use the logged data produced by the algorithms to confirm the empirical performance of the proposed estimators. To check the performance, we calculate the estimation error between the estimates of the value of the uniformly random sampling policy estimated from the dataset obtained by the Thompson sampling and the observed average reward of the uniformly random policy. More details are shown in Appendix~\ref{appdx:det_exp2}. We apply the PBA2IPW, EBA2IPW, EBA2IPW', BAdaIPW, and AdaDM estimators. The results are shown in Table~\ref{tbl:exp_table5} and Figure~\ref{fig1:box_plot}. We show the Bias, MSE, and averaged confidence intervals. While the PBA2IPW and BAdaIPW estimators suffer the high variance, the EBA2IPW, EBA2IPW', and AdaDM estimators show the preferable performances. Although the AdaDM estimator also shows the effectiveness for this dataset, the BA2IPW estimator is theoretically more robust because it is consistent even if $\hat{f}_t$ does not converge to $f^*$.

\begin{table}[h]
\begin{center}
\caption{Results of the CyberAgent dataset. } 
\medskip
\label{tbl:exp_table5}
\scalebox{0.73}[0.73]{
\begin{tabular}{lrrrrrrrrrrrrrrrr}
\toprule
{} &      Bias &      MSE &   $\frac{1.96\times \mathrm{variance}}{\mathrm{sample}\ \mathrm{size}}$\\
\midrule
PBA2IPW &  -0.01766 &  8.25830e-02 &  8.25830e-02  \\
EBA2IPW &  0.00019 &  4.52574e-05 &  3.58261e-03  \\
EBA2IPW' &  0.00095 &  4.03248e-05 &  2.92494e-03  \\
BAdaIPW &  0.14650 &  5.70524e-02 &  7.68610e-02  \\
AdaDM &  0.00714 &  2.26107e-04 &  2.23604e-05  \\
\bottomrule
\end{tabular}
} 
\end{center}
\vspace{-0.3cm}
\end{table}

\begin{figure}[h]
\begin{center}
 \includegraphics[width=55mm]{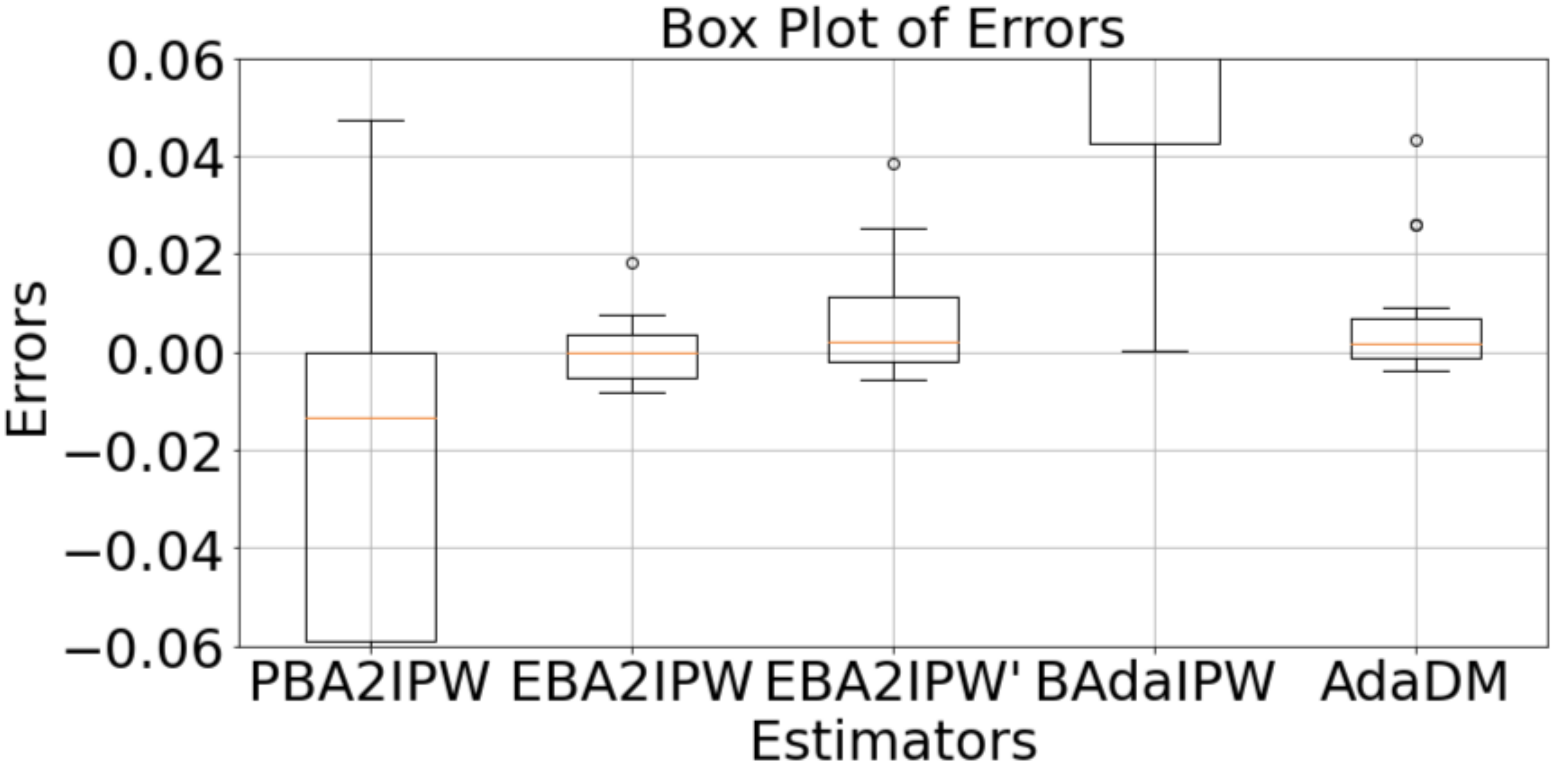}
\end{center}
\caption{Box plot of the CyberAgent dataset.}
\label{fig1:box_plot}
\vspace{-0.3cm}
\end{figure}

\section{Conclusion}
This study presented solutions for causal inference from dependent samples obtained via a batch-based bandit algorithm. By using the asymptotic property in batch, we applied the CLT for an MDS to obtain the asymptotic normality without requiring the convergence of the assignment probability. Additionally, we showed that the proposed method is applicable when the support of arms is incomplete, which is a notorious problem in OPE. In experiments, the proposed batch-based estimators showed theoretically expected performances for benchmark and real-world datasets. 

\bibliographystyle{icml2019}
\bibliography{arxiv}

\clearpage

\onecolumn

\appendix

\setcounter{secnumdepth}{3}

\section{Preliminaries}
\label{sec:prelim}

\subsection{Mathematical Tools}

\begin{proposition}\label{prp:rules_for_ld}[Slutsky Theorem, \citet{greene2003econometric}, Theorem D.~16 1, p.~1117]  If $a_n\xrightarrow{d}a$ and $b_n\xrightarrow{p} b$, then
\begin{align*}
a_nb_n\xrightarrow{d} ba.
\end{align*}
\end{proposition}

\begin{definition}\label{dfn:uniint}[Uniformly Integrable, \citet{GVK126800421}, p.~191]  A sequence $\{A_t\}$ is said to be uniformly integrable if for every $\epsilon > 0$ there exists a number $c>0$ such that 
\begin{align*}
\mathbb{E}[|A_t|\cdot \mathbbm{1}[|A_t \geq c|]] < \epsilon
\end{align*}
for all $t$.
\end{definition}

\begin{proposition}\label{prp:suff_uniint}[Sufficient Conditions for Uniformly Integrable, \citet{GVK126800421}, Proposition~7.7, p.~191]  (a) Suppose there exist $r>1$ and $M<\infty$ such that $\mathbb{E}[|A_t|^r]<M$ for all $t$. Then $\{A_t\}$ is uniformly integrable. (b) Suppose there exist $r>1$ and $M < \infty$ such that $\mathbb{E}[|b_t|^r]<M$ for all $t$. If $A_t = \sum^\infty_{j=-\infty}h_jb_{t-j}$ with $\sum^\infty_{j=-\infty}|h_j|<\infty$, then $\{A_t\}$ is uniformly integrable.
\end{proposition}

\begin{proposition}[$L^r$ Convergence Theorem, \citet{loeve1977probability}]
\label{prp:lr_conv_theorem}
Let $0<r<\infty$, suppose that $\mathbb{E}\big[|a_n|^r\big] < \infty$ for all $n$ and that $a_n \xrightarrow{\mathrm{p}}a$ as $n\to \infty$. The following are equivalent: 
\begin{description}
\item{(i)} $a_n\to a$ in $L^r$ as $n\to\infty$;
\item{(ii)} $\mathbb{E}\big[|a_n|^r\big]\to \mathbb{E}\big[|a|^r\big] < \infty$ as $n\to\infty$; 
\item{(iii)} $\big\{|a_n|^r, n\geq 1\big\}$ is uniformly integrable.
\end{description}
\end{proposition}

\subsection{Martingale Limit Theorems}

\begin{proposition}
\label{prp:mrtgl_WLLN}[Weak Law of Large Numbers for Martingale, \citet{hall2014martingale}]
Let $\{S_n = \sum^{n}_{i=1} X_i, \Omega_{t}, t\geq 1\}$ be a martingale and $\{b_n\}$ a sequence of positive constants with $b_n\to\infty$ as $n\to\infty$. Then, writing $X_{ni} = X_i\mathbbm{1}[|X_i|\leq b_n]$, $1\leq i \leq n$, we have that $b^{-1}_n S_n \xrightarrow{\mathrm{p}} 0$ as $n\to \infty$ if 
\begin{description}
\item[(i)] $\sum^n_{i=1}P(|X_i| > b_n)\to 0$;
\item[(ii)] $b^{-1}_n\sum^n_{i=1}\mathbb{E}[X_{ni}\mid \Omega_{t-1}] \xrightarrow{\mathrm{p}} 0$, and;
\item[(iii)] $b^2_n \sum^n_{i=1}\big\{\mathbb{E}[X^2_{ni}] - \mathbb{E}\big[\mathbb{E}\big[X_{ni}\mid \Omega_{t-1}\big]\big]^2\big\}\to 0$.
\end{description}
\end{proposition}
\begin{remark} The weak law of large numbers for martingale holds when the random variable is bounded by a constant.
\end{remark}

\begin{proposition}
\label{prp:marclt}[Central Limit Theorem for a Martingale Difference Sequence, \citet{GVK126800421}, Proposition~7.9, p.~194] Let $\{\bm{B}_t\}^\infty_{t=1}$ be an $n$-dimensional vector martingale difference sequence with $\overline{\bm{B}}_T=\frac{1}{T}\sum^T_{t=1}\bm{B}_t$. Suppose that (a) $\mathbb{E}[\bm{B}_t\bm{B}^{\top}_t] = \Omega_t$, a positive definite matrix with $(1/T)\sum^T_{t=1}\Omega_t\to\Omega$, a positive definite matrix; (b) $\mathbb{E}[B_{it}B_{jt}B_{it}B_{mt}] < \infty$ for all $t$ and all $i,j,l$, and $m$ (including $i=j=l=m$), where $B_{it}$ is the $i$-th element of vector $\bm{B}_t$; and (c) $(1/T)\sum^{T}_{t=1}\bm{B}_t\bm{B}^{\top}_t\xrightarrow{p}\Omega$. Then $\sqrt{T}\overline{\bm{B}}_T\xrightarrow{d}\mathcal{N}(\bm{0}, \Omega)$.
\end{proposition}

\section{Generalized Method of Martingale Difference Moments}
\label{appdx:sec:gmmdm}
In this section, using martingale difference sequences, we establish a frameworks of GMM with samples with dependency. Let $s_t$ be a sample with the domain $\mathcal{S}$, $\Theta$ be the space of a parameter $\theta$, and $M$ be the number of moment conditions. For a vector-valued function $\bm{h}_t:\mathcal{S}\times\Theta \to \mathbb{R}^M$ and a parameter $\theta_0\in\Theta$, let $\{\bm{h}_t(s_t; \theta_0)\}^{T}_{t=1}$ be a martingale difference sequence, i.e., 
\begin{align}
\label{asm:mds}
\mathbb{E}\big[\bm{h}_t(S_t; \theta_0) \mid \Omega_{t-1}\big] = 0.
\end{align}
Then, let $\bm{q}(\theta)$ be the moment condition defined as $\bm{q}(\theta) = \mathbb{E}\left[\frac{1}{T}\sum^T_{t=1}\bm{h}_t(S_t; \theta)\right]$. For the parameter $\theta_0\in\Theta$, the moment condition $\bm{q}(\theta_0; S_t)$ is zero from (\ref{asm:mds}), i.e.,
\begin{align*}
\bm{q}(\theta_0) = \mathbb{E}\left[\frac{1}{T}\sum^T_{t=1}\bm{h}_t(S_t; \theta_0)\right] = 0.
\end{align*}
Using samples $\big\{s_t\big\}^T_{t=1}\in\mathcal{S}$, we approximate the moment condition by $\hat{\bm{q}}(\theta) = \frac{1}{T}\bm{h}_t(s_t; \theta)$. Then, we define a GMM-like estimator $\hat{\theta}_T$ as follows:
\begin{align*}
\hat{\theta}_T = \argmin_{\theta\in\Theta} \left(\hat{\bm{q}}(\theta)\right)^\top \hat{W}_T \left(\hat{\bm{q}}(\theta)\right),
\end{align*}
where $\hat{W}_T$ is a positive definite weight matrix constructed from $T$ samples. Compared with the standard GMM, we do not assume that the samples are not i.i.d. However, from the assumption that $\{\bm{h}_t(s_t; \theta_0)\}^{T}_{t=1}$ is a martingale difference sequence, we can derive the following results on the consistency and asymptotic normality of $\hat{\theta}_T$ under appropriate regularity conditions. We refer this method as \emph{Generalized Method of Martingale Difference Moments} (GMMDM).

For brevity, let us denote $\widehat{R}^{\mathrm{BA2IPW}}_T(\epol)$ as $\hat{\theta}^{\mathrm{OPE}}_T$. Using the sequence $\left\{\bm{h}^{\mathrm{OPE}}_t\left(X_t, A_t, Y_t; \theta, \hat{f}_{t-1}, \epol\right)\right\}^{T}_{t=1}$, we define an estimator of OPE as
\begin{align*}
\hat{\theta}^{\mathrm{OPE}}_T = \argmin_{\theta\in\Theta} \left(\hat{\bm{q}}^{\mathrm{OPE}}_T(\theta)\right)^\top \hat{W}_T \left(\hat{\bm{q}}^{\mathrm{OPE}}_T(\theta)\right),
\end{align*}
where $\hat{\bm{q}}^{\mathrm{OPE}}_T(\theta) = \frac{1}{T}\sum^{T}_{t=1}\bm{h}^{\mathrm{OPE}}_t\left(X_t, A_t, Y_t; \theta, \hat{f}_{t-1}, \epol\right)$. Note that we can consider that the estimator defined in (\ref{gmmdm_ope}) is an application of GMMDM with the moment condition
\begin{align*}
\bm{q}^{\mathrm{OPE}}(\theta_0) = \mathbb{E}\left[\frac{1}{T}\sum^T_{t=1}\bm{h}^{\mathrm{OPE}}_t\left(X_t, A_t, Y_t; \theta_0, \hat{f}_{t-1}, \epol\right)\right]=0.
\end{align*}
For the minimization problem defined in (\ref{gmmdm_ope}), we can analytically calculate the minimizer as
\begin{align*}
\hat{\theta}^{\mathrm{OPE}}_T  = \big(I^\top \hat{W}_T I\big)^{-1}I^\top \hat{W}_T D_T(\epol),
\end{align*}
where $I$ is an $M$-dimensional vector such that $I = \big(1\ 1\ \cdots\ 1\big)^\top$ and $D_T(\epol)$ is
\begin{align*}
\begin{pmatrix}
\frac{1}{t_{1}}\sum^{t_{1}}_{t=1} \eta_t(x, k, y; \tau, \theta, f, \epol) \\
\frac{1}{t_{2} - {t_{1}}}\sum^{t_2}_{t=t_1+1} \eta_t(x, k, y; \tau, \theta, f, \epol)\\
\vdots\\
\frac{1}{T-t_{M-1}}\sum^{T}_{t=t_{M-1}+1} \eta_t(x, k, y; \tau, \theta, f, \epol)\\
\end{pmatrix}.
\end{align*}

\section{Proof of Theorem~\ref{thm:main}}
\label{appdx:main}
Instead of $\hat{\theta}^{\mathrm{OPE}}_T = w_T D_T(\epol)$, from the original formulation Eq.~\eqref{gmmdm_ope}, we consider an estimator $\hat{\theta}^{\mathrm{OPE}}_T = \big(I^\top \hat W_T I\big)^{-1}I^\top \hat W_T D_T(\epol)$, where $W_T$ is a $(M\times M)$-dimensional positive-definite matrix. Let us note that $w_T = \big(I^\top \hat W_T I\big)^{-1}I^\top \hat W_T$. We prove the following theorem, which is a generalized statement of Theorem~\ref{appdx:main}.

\begin{theorem}[Asymptotic Distribution the BA2IPW Estimator]
\label{thm:main}
Suppose that 
\begin{description}
\item[(i)] $\hat W_T \xrightarrow{\mathrm{p}} W$;
\item[(ii)] $W$ is a positive definite;
\item[(iii)] $\hat{f}_{t-1}(a, x) \xrightarrow{\mathrm{p}} f^*(a, x)$ for $m > 0$.
\item[(iv)] There exists a constant $C_f > 0$ such that $\big|\hat{f}_{t-1}(a, x)\big| < C_f$ for $\tau \in I$.  
\end{description}
Then, under Assumptions~\ref{asm:overlap_pol} and \ref{asm:overlap_outcome}, 
\begin{align*}
\sqrt{T}\big(\widehat{R}^{\mathrm{BA2IPW}}_T - R(\epol)\big)  \xrightarrow{\mathrm{d}}\mathcal{N}\big(0, \sigma^2\big),
\end{align*}
where $\sigma^2 = \big(I^\top W I \big)^{-1} I^\top W \Omega W^\top I \big(I^\top W I \big)^{-1}$ and $\Omega$ is a $(M\times M)$ diagonal matrix such that the $(\tau\times \tau)$-element is $\frac{1}{r_\tau}\mathbb{E}\Big[\sum^{K}_{a=1}\left\{\frac{\big(\epol(a\mid X)\big)^2\nu^*(a, X)}{\pi_{\tau}(a\mid X, \Omega_{t_{\tau - 1}})} + \Big(\epol(a\mid X)f^*(a, X) - R(\epol)\Big)^2\right\}\Big]$.
\end{theorem}

\begin{proof}
For
\begin{align*}
\sqrt{T}\big(\hat{\theta}^{\mathrm{OPE}}_T - \theta_0\big)  = \big(I^\top \hat W_T I\big)^{-1}I^\top \hat W_T \sqrt{T} \hat{\bm{q}}^{\mathrm{OPE}}_T(\theta_0),
\end{align*}
where
\begin{align*}
&\hat{\bm{q}}^{\mathrm{OPE}}_T(\theta_0) = \\
&\begin{pmatrix}
\frac{1}{T}\frac{1}{r_1}\sum^{T}_{t=1} \left(\sum^{K}_{a=1}\left\{\frac{\epol(a\mid X_t)\mathbbm{1}[A_t=a]\big\{Y_t - \hat{f}_{t-1}(a, X_t)\big\}}{\pi_{1}(a\mid x, \Omega_{0})} + \epol(a\mid X_t)\hat{f}_{t-1}(a, X_t)\right\} - \theta_0\right)\mathbbm{1}\big[t_0=0 < t \leq t_{1} \big]\\
\frac{1}{T}\frac{1}{r_2}\sum^{T}_{t=1} \left(\sum^{K}_{a=1}\left\{\frac{\epol(a\mid X_t)\mathbbm{1}[A_t=a]\big\{Y_t - \hat{f}_{t-1}(a, X_t)\big\}}{\pi_{2}(a\mid X_t, \Omega_{t_{1}})} + \epol(a\mid X_t)\hat{f}_{t-1}(a, X_t)\right\} - \theta_0\right)\mathbbm{1}\big[t_{1} < t \leq t_{2} \big]\\
\vdots\\
\frac{1}{T}\frac{1}{r_M}\sum^{T}_{t=1} \left(\sum^{K}_{a=1}\left\{\frac{\epol(a\mid X_t)\mathbbm{1}[A_t=a]\big\{Y_t - \hat{f}_{t-1}(a, X_t)\big\}}{\pi_{M}(a\mid X_t, \Omega_{t_{M-1}})} + \epol(a\mid X_t)\hat{f}_{t-1}(a, X_t)\right\} - \theta_0\right)\mathbbm{1}\big[t_{M - 1} < t \leq t_{M} = T \big]\\
\end{pmatrix},
\end{align*}
we show that 
\begin{align*}
\sqrt{T} \hat{\bm{q}}^{\mathrm{OPE}}_T(\theta_0) \xrightarrow{\mathrm{d}} \mathcal{N}\big(0, \Omega\big),
\end{align*}
where $\Omega$ is a diagonal matrix such that the $(\tau,\tau)$-element is
\begin{align*}
&\frac{1}{r_\tau}\mathbb{E}\left[\sum^{K}_{a=1}\frac{\big(\epol(a\mid X)\big)^2\mathrm{Var}(Y(a)\mid X)}{\pi_{\tau}(a\mid X, \Omega_{t_{\tau - 1}})} + \left(\sum^{K}_{a=1}\epol(a\mid X)\mathbb{E}\big[Y(a)\mid X\big] - \theta_0\right)^2 \right].
\end{align*}
Then, from Slutsky Theorem (Proposition~\ref{prp:rules_for_ld} in Appendix), we can show that 
\begin{align*}
\big(I^\top \hat W_T I\big)^{-1}I^\top \hat W_T \sqrt{T} \hat{\bm{q}}^{\mathrm{OPE}}_T(\theta_0) \xrightarrow{\mathrm{d}}\mathcal{N}\big(0, \big(I^\top W I \big)^{-1} I^\top W \Omega W^\top I \big(I^\top W I \big)^{-1}\big).
\end{align*}

To show this result, we use the central limit theorem for martingale difference sequences (Proposition~\ref{prp:marclt} in Appendix~\ref{sec:prelim}) by checking the following conditions:
\begin{description}
\item[(a)] $(1/T)\sum^T_{t=1}\Omega_t\to\Omega$, where $\Omega_t = \mathbb{E}\left[\left(\bm{h}^{\mathrm{OPE}}_t\left(X_t, A_t, Y_t; \theta_0, f_{t-1}, \epol\right)\right)\left(\bm{h}^{\mathrm{OPE}}_t\left(X_t, A_t, Y_t; \theta_0, f_{t-1}, \epol\right)\right)^{\top}\right]$;
\item[(b)] $\mathbb{E}\big[\tilde{h}_t(i, \theta_0, f_{t-1}, \epol)\tilde{h}_t(j, \theta_0, f_{t-1}, \epol)\tilde{h}_t(a, \theta_0, f_{t-1}, \epol)\tilde{h}_t(l, \theta_0, f_{t-1}, \epol)\big] < \infty$ for $i,j,k,l\in I$, where $\tilde{h}_t(a, \theta_0, f_{t-1}, \epol)=h^\mathrm{OPE}_t(X_t, A_t, Y_t; k, \theta_0, f_{k}, \epol)$ for $k\in I$;
\item[(c)] $\frac{1}{T}\sum^{T}_{t=1}\left(\bm{h}^{\mathrm{OPE}}_t\left(X_t, A_t, Y_t; \theta_0, f_{t-1}, \epol\right)\right)\left(\bm{h}^{\mathrm{OPE}}_t\left(X_t, A_t, Y_t; \theta_0, f_{t-1}, \epol\right)\right)^{\top}\xrightarrow{p}\Omega$,
\end{description}

\subsection*{Step~1: Condition~(a)}
From 
\begin{align*}
&\Omega_t = \mathbb{E}\left[\left(\bm{h}^{\mathrm{OPE}}_t\left(X_t, A_t, Y_t; \theta_0, f_{t-1}, \epol\right)\right)\left(\bm{h}^{\mathrm{OPE}}_t\left(X_t, A_t, Y_t; \theta_0, f_{t-1}, \epol\right)\right)^{\top}\right],
\end{align*}
the matrix $(1/T)\sum^T_{t=1}\Omega_t$ becomes a diagonal matrix such that the $(\tau, \tau)$-element is 
\begin{align*}
&\frac{1}{r^2_\tau T}\sum^T_{t=1}\mathbb{E}\left[\left(\sum^{K}_{a=1}\left\{\frac{\epol(a\mid X_t)\mathbbm{1}[A_t=a]\big\{Y_t - f_{t- 1}(a, X_t)\big\}}{\pi_{\tau}(a\mid X_t, \Omega_{t_{\tau - 1}})} + \epol(a\mid X_t)f_{t-1}(a, X_t)\right\} - \theta_0\right)^2\mathbbm{1}\big[t_{\tau - 1} < t \leq t_{\tau} \big]\right].
\end{align*}

For $\tau \in I$ and $t$ such that $t_{\tau - 1} < t \leq t_{\tau}$,
\begin{align*}
&\mathbb{E}\left[\left(\sum^{K}_{a=1}\left\{\frac{\epol(a\mid X_t)\mathbbm{1}[A_t=a]\big\{Y_t - f_{t-1}(a, X_t)\big\}}{\pi_{\tau}(a\mid X_t, \Omega_{t_{\tau - 1}})} + \epol(a\mid X_t)f_{t-1}(a, X_t)\right\} - \theta_0\right)^2\right]\\
&- \mathbb{E}\left[\left(\sum^{K}_{a=1}\left\{\frac{\epol(a\mid X_t)\mathbbm{1}[A_t=a]\big\{Y_t - \mathbb{E}\big[Y_t(a)\mid X_t\big]\big\}}{\pi_{\tau}(a\mid X_t, \Omega_{t_{\tau - 1}})} + \epol(a\mid X_t)\mathbb{E}\big[Y(a)\mid X_t\big]\right\} - \theta_0\right)^2\right]\\
&\leq \mathbb{E}\Bigg[\Bigg|\left(\sum^{K}_{a=1}\left\{\frac{\epol(a\mid X_t)\mathbbm{1}[A_t=a]\big\{Y_t - f_{t-1}(a, X_t)\big\}}{\pi_{\tau}(a\mid X_t, \Omega_{t_{\tau - 1}})} + \epol(a\mid X_t)f_{t-1}(a, X_t)\right\} - \theta_0\right)^2\\
&\ \ \ \ \ \ \ \ - \left(\sum^{K}_{a=1}\left\{\frac{\epol(a\mid X_t)\mathbbm{1}[A_t=a]\big\{Y_t - \mathbb{E}\big[Y_t(a)\mid X_t\big]\big\}}{\pi_{\tau}(a\mid X_t, \Omega_{t_{\tau - 1}})} + \epol(a\mid X_t)\mathbb{E}\big[Y(a)\mid X_t\big]\right\} - \theta_0\right)^2\Bigg|\Bigg]
\end{align*}
Because $\alpha^2 - \beta^2 = (\alpha + \beta)(\alpha - \beta)$, there exists a constant $\gamma_0 > 0$ such that
\begin{align*}
&\leq \gamma_0\mathbb{E}\Bigg[\Bigg|\sum^{K}_{a=1}\Bigg\{\frac{\epol(a\mid X_t)\mathbbm{1}[A_t=a]\big\{Y_t - f_{t-1}(a, X_t)\big\}}{\pi_{\tau}(a\mid X_t, \Omega_{t_{\tau - 1}})} + \epol(a\mid X_t)f_{t-1}(a, X_t)\\
&\ \ \ \ \ \ \ \ \ \ \ \ \ \ - \frac{\epol(a\mid X_t)\mathbbm{1}[A_t=a]\big\{Y_t - \mathbb{E}\big[Y_t(a)\mid X_t\big]\big\}}{\pi_{\tau}(a\mid X_t, \Omega_{t_{\tau - 1}})} - \epol(a\mid X_t)\mathbb{E}\big[Y(a)\mid X_t\big]\Bigg\}\Bigg|\Bigg]
\end{align*}
Then, there exist constants $\gamma_1 > 0$ such that
\begin{align*}
\leq &\gamma_1 \mathbb{E}\left[\sum^{K}_{a=1}\Big|f_{t-1}(a, X_t) - \mathbb{E}\big[Y(a)\mid X_t\big]\Big|\right].
\end{align*}
Here, from the assumption that $f_{t-1}(a, x)- \mathbb{E}\big[Y(a)\mid X\big]\xrightarrow{\mathrm{p}}0$ for $\tau=2,3,\dots,M$, and $f_{t_{\tau-1}}(a, x)$ is bounded for $\tau \in I$, we can use $L^r$ convergence theorem (Proposition~\ref{prp:lr_conv_theorem} in Appendix~\ref{sec:prelim}). First, to use $L^r$ convergence theorem, we use boundedness of $f_{t_m}$ and Proposition~\ref{prp:suff_uniint} to derive the uniform integrability of $f_{t_m}$ for $m=0,1,\dots,\tau-1$. Then, from $L^r$ convergence theorem, we have $\mathbb{E}\big[|f_{t_m}(a, X) - \mathbb{E}[Y(a)\mid X]|\big]\to 0$ as $t_m\to\infty$. Using this results, we can show that, as $t_{\tau-1}\to\infty$  (this also means $T\to\infty$),
\begin{align*}
&\gamma_1 \sum^{K}_{a=1}\mathbb{E}\left[\Big|f_{t-1}(a, X_t) - \mathbb{E}\big[Y(a)\mid X_t\big]\Big|\right]\\
&\to 0.
\end{align*}
Therefore, as $t_{\tau-1}\to\infty$ ($T\to\infty$),
\begin{align*}
&\mathbb{E}\left[\left(\sum^{K}_{a=1}\left\{\frac{\epol(a\mid X_t)\mathbbm{1}[A_t=a]\big\{Y_t - f_{t-1}(a, X_t)\big\}}{\pi_{\tau}(a\mid X_t, \Omega_{t_{\tau - 1}})} + \epol(a\mid X_t)f_{t-1}(a, X_t)\right\} - \theta_0\right)^2\right]\\
&\to \mathbb{E}\left[\left(\sum^{K}_{a=1}\left\{\frac{\epol(a\mid X_t)\mathbbm{1}[A_t=a]\big\{Y_t - \mathbb{E}\big[Y_t(a)\mid X_t\big]\big\}}{\pi_{\tau}(a\mid X_t, \Omega_{t_{\tau - 1}})} + \epol(a\mid X_t)\mathbb{E}\big[Y(a)\mid X_t\big]\right\} - \theta_0\right)^2\right].
\end{align*} 
Then, by using $\mathbbm{1}[A_t=a]\mathbbm{1}[A_t=l] = 0$, $\mathbb{E}\left[\frac{\mathbbm{1}[A_t=a]Y^2_t}{\big(\pi_{\tau}(a\mid X_t, \Omega_{t_{\tau - 1}})\big)^2}\right] = \mathbb{E}\left[\frac{\mathbb{E}\big[Y^2_t(a)\mid X_t\big]}{\pi_{\tau}(a\mid X_t, \Omega_{t_{\tau - 1}})}\right]$, and $\frac{1}{r_\tau T}\sum^T_{t=1}\mathbbm{1}\big[t_{\tau - 1} < t \leq t_{\tau} \big]=1$,
\begin{align*}
&\mathbb{E}\left[\left(\sum^{K}_{a=1}\left\{\frac{\epol(a\mid X_t)\mathbbm{1}[A_t=a]\big\{Y_t - \mathbb{E}\big[Y_t(a)\mid X_t\big]\big\}}{\pi_{\tau}(a\mid X_t, \Omega_{t_{\tau - 1}})} + \epol(a\mid X_t)\mathbb{E}\big[Y(a)\mid X_t\big]\right\} - \theta_0\right)^2\right]\\
&= \mathbb{E}\left[\sum^{K}_{a=1}\left\{\frac{\big(\epol(a\mid X_t)\big)^2\mathrm{Var}\big(Y_t(a) \mid X_t\big)}{\pi_{\tau}(a\mid X_t, \Omega_{t_{\tau - 1}})} + \Big(\epol(a\mid X_t)\mathbb{E}\big[Y_t(a)\mid X_t\big] - \theta_0\Big)^2\right\}\right].
\end{align*}
In addition, the variance does not depend on $t$. We represent the independence by omitting the subscript $t$, i.e.,
\begin{align*}
&\mathbb{E}\left[\sum^{K}_{a=1}\left\{\frac{\big(\epol(a\mid X_t)\big)^2\mathrm{Var}\big(Y_t(a) \mid X_t\big)}{\pi_{\tau}(a\mid X_t, \Omega_{t_{\tau - 1}})} + \Big(\epol(a\mid X_t)\mathbb{E}\big[Y_t(a)\mid X_t\big] - \theta_0\Big)^2\right\}\right]\\
&=\mathbb{E}\left[\sum^{K}_{a=1}\left\{\frac{\big(\epol(a\mid X)\big)^2\mathrm{Var}\big(Y(a) \mid X\big)}{\pi_{\tau}(a\mid X, \Omega_{t_{\tau - 1}})} + \Big(\epol(a\mid X)\mathbb{E}\big[Y_t(a)\mid X\big] - \theta_0\Big)^2\right\}\right].
\end{align*}
Therefore, we have
\begin{align*}
&\frac{1}{r^2_\tau T}\sum^T_{t=1}\mathbb{E}\left[\left(\sum^{K}_{a=1}\left\{\frac{\epol(a\mid X_t)\mathbbm{1}[A_t=a]\big\{Y_t - f_{t-1}(a, X_t)\big\}}{\pi_{\tau}(a\mid X_t, \Omega_{t_{\tau - 1}})} + \epol(a\mid X_t)f_{t-1}(a, X_t)\right\} - \theta_0\right)^2\mathbbm{1}\big[t_{\tau - 1} < t \leq t_{\tau} \big]\right]\\
&\to \frac{1}{r_\tau}\mathbb{E}\left[\sum^{K}_{a=1}\left\{\frac{\big(\epol(a\mid X)\big)^2\mathrm{Var}\big(Y(a) \mid X\big)}{\pi_{\tau}(a\mid X, \Omega_{t_{\tau - 1}})} + \Big(\epol(a\mid X)\mathbb{E}\big[Y(a)\mid X\big] - \theta_0\Big)^2\right\}\right].
\end{align*}

Thus, the matrix $(1/T)\sum^T_{t=1}\Omega_t$ converges to a diagonal matrix $\Omega$ as $T\to\infty$, where the $(\tau,\tau)$-element of $\Omega$ is
\begin{align*}
\frac{1}{r_\tau}\mathbb{E}\left[\sum^{K}_{a=1}\left\{\frac{\big(\epol(a\mid X)\big)^2\mathrm{Var}\big(Y(a) \mid X\big)}{\pi_{\tau}(a\mid X, \Omega_{t_{\tau - 1}})} + \Big(\epol(a\mid X)\mathbb{E}\big[Y(a)\mid X\big] - \theta_0\Big)^2\right\}\right].
\end{align*}

\subsection*{Step~2: Condition~(b)}
Because we assume that all variables are bounded, this condition holds.

\subsection*{Step~3: Condition~(c)}
Here, we check that $(1/T)\sum^{T}_{t=1}\left(\bm{h}^{\mathrm{OPE}}_t\left(X_t, A_t, Y_t; \theta_0, f_{t-1}, \epol\right)\right)\left(\bm{h}^{\mathrm{OPE}}_t\left(X_t, A_t, Y_t; \theta_0, f_{t-1}, \epol\right)\right)^{\top}\xrightarrow{p}\Omega$. The $(\tau,\tau)$-element of the matrix is 
\begin{align*}
&\frac{1}{T}\sum^{T}_{t=1}\frac{1}{r^2_\tau}\left(\sum^{K}_{a=1}\left\{\frac{\epol(a\mid X_t)\mathbbm{1}[A_t=a]\big\{Y_t - f_{t-1}(a, X_t)\big\}}{\pi_{\tau}(a\mid X, \Omega_{t_{\tau - 1}})} + \epol(a\mid X_t)f_{t-1}(a, X_t)\right\} - \theta\right)^2\mathbbm{1}\big[t_{\tau - 1} < t \leq t_{\tau} \big]\\
&=\frac{1}{T}\sum^{T}_{t=1}\frac{1}{r^2_\tau}\left(\sum^{K}_{a=1}\left\{\frac{\epol(a\mid X_t)\mathbbm{1}[A_t=a]\big\{Y_t - f_{t-1}(a, X_t)\big\}}{\pi_{\tau}(a\mid X, \Omega_{t_{\tau - 1}})} + \epol(a\mid X_t)f_{t-1}(a, X_t)\right\} - \theta\right)^2\mathbbm{1}\big[t_{\tau - 1} < t \leq t_{\tau} \big]\\
& - \frac{1}{T}\sum^{T}_{t=1}\frac{1}{r^2_\tau}\left(\sum^{K}_{a=1}\left\{\frac{\epol(a\mid X_t)\mathbbm{1}[A_t=a]\big\{Y_t - \mathbb{E}\big[Y(a)\mid X_t\big]\big\}}{\pi_{\tau}(a\mid X, \Omega_{t_{\tau - 1}})} + \epol(a\mid X_t)\mathbb{E}\big[Y(a)\mid X_t\big]\right\} - \theta\right)^2\mathbbm{1}\big[t_{\tau - 1} < t \leq t_{\tau} \big]\\
& + \frac{1}{T}\sum^{T}_{t=1}\frac{1}{r^2_\tau}\left(\sum^{K}_{a=1}\left\{\frac{\epol(a\mid X_t)\mathbbm{1}[A_t=a]\big\{Y_t - \mathbb{E}\big[Y(a)\mid X_t\big]\big\}}{\pi_{\tau}(a\mid X, \Omega_{t_{\tau - 1}})} + \epol(a\mid X_t)\mathbb{E}\big[Y(a)\mid X_t\big]\right\} - \theta\right)^2\mathbbm{1}\big[t_{\tau - 1} < t \leq t_{\tau} \big].
\end{align*}
The part 
\begin{align*}
&\frac{1}{T}\sum^{T}_{t=1}\frac{1}{r^2_\tau}\left(\sum^{K}_{a=1}\left\{\frac{\epol(a\mid X_t)\mathbbm{1}[A_t=a]\big\{Y_t - f_{t-1}(a, X_t)\big\}}{\pi_{\tau}(a\mid X, \Omega_{t_{\tau - 1}})} + \epol(a\mid X_t)f_{t-1}(a, X_t)\right\} - \theta\right)^2\mathbbm{1}\big[t_{\tau - 1} < t \leq t_{\tau} \big]\\
& - \frac{1}{T}\sum^{T}_{t=1}\frac{1}{r^2_\tau}\left(\sum^{K}_{a=1}\left\{\frac{\epol(a\mid X_t)\mathbbm{1}[A_t=a]\big\{Y_t - \mathbb{E}\big[Y(a) \mid X_t\big]\big\}}{\pi_{\tau}(a\mid X, \Omega_{t_{\tau - 1}})} + \epol(a\mid X_t)\mathbb{E}\big[Y(a)\mid X_t\big]\right\} - \theta\right)^2\mathbbm{1}\big[t_{\tau - 1} < t \leq t_{\tau} \big]
\end{align*}
converges in probability to $0$ because $f_{t-1}(a, X_t) \xrightarrow{\mathrm{p}} \mathbb{E}\big[Y(a)\mid X_t\big]$. The term 
\begin{align*}
&\frac{1}{T}\sum^{T}_{t=1}\frac{1}{r^2_\tau}\left(\sum^{K}_{a=1}\left\{\frac{\epol(a\mid X_t)\mathbbm{1}[A_t=a]\big\{Y_t - \mathbb{E}\big[Y(a) \mid X_t\big]\big\}}{\pi_{\tau}(a\mid X, \Omega_{t_{\tau - 1}})} + \epol(a\mid X_t)\mathbb{E}\big[Y(a) \mid X_t\big]\right\} - \theta_0\right)^2\mathbbm{1}\big[t_{\tau - 1} < t \leq t_{\tau} \big].
\end{align*}
converges in probability to 
\begin{align*}
\frac{1}{r_\tau}\mathbb{E}\left[\sum^{K}_{a=1}\left\{\frac{\big(\epol(a\mid X)\big)^2\mathrm{Var}\big(Y(a) \mid X\big)}{\pi_{\tau}(a\mid X, \Omega_{t_{\tau - 1}})} + \Big(\epol(a\mid X)\mathbb{E}\big[Y(a)\mid X\big] - \theta_0\Big)^2\right\}\right].
\end{align*}
from the weak law of large numbers for i.i.d. samples as $t_{\tau-1}-t_{\tau} \to \infty$ because the samples are i.i.d. between $t_{\tau-1}$ and $t_{\tau}$. 
\end{proof}

\section{Necessity of Martingale Difference Sequences}
\label{appdx:nec_mds}
In the proposed method, we construct a moment condition using martingale difference sequences. On the other hand, for some readers, using martingale difference sequences may look unnecessary because samples are i.i.d in each block between $t_{\tau-1}$ and $t_{\tau}$. Therefore, such readers also might feel that we can use $f_T(a, x)$, which is an estimator of $\mathbb{E}[Y(a)\mid x]$ using samples until $T$-th period, without going through constructing several estimators $\{f_{t_\tau}\}^{M-1}_{\tau=0}$. However, in that case, it is difficult to guarantee the asymptotic normality of the proposed estimator. For example, we can consider Cram\'{e}r-Wold theorem, which is stated as follows.
\begin{proposition}
Let $\bm{R}_T$ and $\bm{R}$ be $k$-dimensional random vectors. Then, for all $\bm{v}\in\mathbb{R}^k$, 
\begin{align*}
\bm{R}_T\xrightarrow{\mathrm{d}} \bm{R} \Leftrightarrow \braket{\bm{v}, \bm{R}_T} \xrightarrow{\mathrm{d}} \braket{\bm{v}, \bm{R}}.
\end{align*}
\end{proposition}
Let $\tilde{\bm{h}}_t\left(\theta_0\right)$ be
\begin{align*}
\begin{pmatrix} 
\frac{1}{r_1}\left(\eta_t(x, k, y; \tau, \theta, f, \epol) - \theta_0\right)\mathbbm{1}\big[t_{\tau - 1} < t \leq t_{1} \big] \\
\frac{1}{r_2}\left(\eta_t(x, k, y; \tau, \theta, f, \epol) - \theta_0\right)\mathbbm{1}\big[t_{1} < t \leq t_{2} \big] \\
 \vdots \\ 
\frac{1}{r_M}\left(\eta_t(x, k, y; \tau, \theta, f, \epol) - \theta_0\right)\mathbbm{1}\big[t_{M - 1} < t \leq t_{M} \big]
\end{pmatrix},
\end{align*}
where $f_T(a, x)$ is an estimator of $\mathbb{E}[Y(a)\mid x]$ using samples until $T$-th period, i.e., all samples. Then, we consider the asymptotic property of $\tilde{\bm{q}}_T = \frac{1}{T}\sum^T_{t=1}\tilde{\bm{h}}_t\left(\theta_0\right)$. From Cram\'{e}r-Wold theorem, there exists a random vector $\bm{R}$ such that $\tilde{\bm{q}}_T \xrightarrow{\mathrm{d}} \bm{R}$ if and only if $\braket{\bm{v}, \tilde{\bm{q}}_T} \xrightarrow{\mathrm{d}} \braket{\bm{v}, \tilde{\bm{q}}_T}$. Here, for $\bm{v}=(1\ 1\ \cdots 1)^\top$, we can calculate $\braket{\bm{v}, \tilde{\bm{q}}_T}$ as 
\begin{align*}
&\frac{1}{T}\sum^T_{t=1}\frac{1}{r_\tau}\Bigg(\sum^{K}_{a=1}\Bigg\{\frac{\epol(a\mid X_t)\mathbbm{1}[A_t=a]\big\{Y_t - f_T(a, X_t)\big\}}{\sum_{\tau\in I}\pi_{\tau}(a\mid X_t, \Omega_{t_{\tau - 1}})\mathbbm{1}\big[t_{\tau - 1} < t \leq t_{\tau} \big]}\nonumber\\
&\ \ \ \ \ \ \ \ \ \ \ \ \ \ \ \ \ \ \ \ \ \ \ \ \ \ \ \ \ \ \ \ \ \ \ \ \ \ \ \ \ \ + \epol(a\mid X_t)f_T(a, X_t)\Bigg\} - \theta_0\Bigg).
\end{align*}
Because of the existence of $f_T$, the samples in the sum of the above equation have correlation each other. Therefore, in general, it is difficult to derive the asymptotic distribution of $\braket{\bm{v}, \tilde{\bm{q}}_T}$. As far as we know, it is not guaranteed that there exists the asymptotic distribution of $\braket{\bm{v}, \tilde{\bm{q}}_T}$, and it is an open problem. If there does not exist the asymptotic distribution of $\braket{\bm{v}, \tilde{\bm{q}}_T}$, we also cannot derive the asymptotic distribution of $\tilde{\bm{q}}_T$. 

More intuitively, even thought random variables $B_{1}$ and $B_{2}$ follow normal distribution and they are uncorrelated, it does not guarantee $B_1+B_2$ follows normal distribution when they are dependent.

\section{BA2IPW with Incomplete Support of Actions}
\label{appdx:incomp_support_of_arms}
As an application of BA2IPW, we consider an OPE without Assumption~\ref{asm:overlap_pol}, which assumes that there exists $C_1$ such that $0\leq \frac{\epol(a\mid x)}{p_t(a\mid x)}\leq C_1$. In stead of Assumption~\ref{asm:overlap_pol}, we consider a situation in which we are allowed to change the support of actions in each batch. For example, in the first batch, we choose an action from a set $\{1,2,3\}$ with probability larger than $0$, but we choose an action from a set $\{1,2,4\}$ with probability larger than $0$ in the second batch. In this case, the probability of an action $4$ is $0$ in the first batch while the probability of an action $3$ is $0$ in the second batch. We often face such a kind of situation in practice. For dealing with this situation, instead of Assumption~\ref{asm:overlap_pol}, we put Assumption~\ref{asm:overlap_pol2}. Under this assumption, if $\pi_{\tau}(a\mid x, \Omega_{t_{\tau - 1}})>0$ for at least one batch, we are allowed to use $\pi_{\tau'}(a\mid x, \Omega_{t_{\tau' - 1}})=0$ for $\tau'\neq \tau$. 

For ease of discussion, we assume Assumption~\ref{asm:overlap_pol} for defining an estimator. The following estimator is a generalization of a BA2IPW estimator. Then, we explain the estimator still works only with Assumption~\ref{asm:overlap_pol2} instead of Assumption~\ref{asm:overlap_pol}. We call the estimator \emph{BA2IPW with Incomplete Support} (BA2IPWIS). First, let us define
\begin{align*}
&g_{a, \tau, T} = \frac{1}{T}\sum^{T}_{t=1} \widetilde{\phi}_{a, \tau}(X_t, A_t, Y_t; f, \epol), \mathrm{and}\\
&\widetilde{\phi}_{a, \tau}(x, k, y; f, \epol)\\
&:= \frac{\epol(a\mid x)}{r_\tau}\Bigg\{\frac{\mathbbm{1}[k=a]\big\{y - f(a, x)\big\}}{\pi_{\tau}(a\mid x, \Omega_{t_{\tau - 1}})}+f(a, x)\Bigg\}\\
&\ \ \ \ \ \ \ \ \ \ \ \ \ \ \ \ \ \ \ \ \ \ \ \ \ \ \ \ \ \ \ \ \ \ \ \ \ \ \ \ \ \ \ \ \ \ \ \ \ \ \ \ \times\mathbbm{1}\big[t_{\tau - 1} < t \leq t_{\tau} \big].\nonumber
\end{align*}
Then, let us define 
\begin{align*}
\ddot{h}_T = \begin{pmatrix}
g_{1, 1, T} - \theta^1_0\\
g_{1, 2, T} - \theta^1_0\\
\vdots\\
g_{1, M, T} - \theta^1_0\\
g_{2, 1, T} - \theta^2_0\\
\vdots\\
g_{K-1, M, T} - \theta^{K-1}_0\\
g_{K, 1, T} - \theta^{K}_0\\
\vdots\\
g_{K, M, T} - \theta^{K}_0
\end{pmatrix}\ \mathrm{and}\ 
\widetilde{D}_T = \begin{pmatrix}
g_{1, 1, T} \\
g_{1, 2, T} \\
\vdots\\
g_{1, M, T} \\
g_{2, 1, T} \\
\vdots\\
g_{K-1, M, T} \\
g_{K, 1, T} \\
\vdots\\
g_{K, M, T} 
\end{pmatrix},
\end{align*}
As well as Theorem~\ref{thm:main}, we have
\begin{align*}
\sqrt{N} \ddot{h}_T \xrightarrow{\mathrm{d}} \mathcal{N}(\bm{0}, \widetilde{\Sigma}),
\end{align*}
where $\widetilde{\Sigma}$ is the $(KM\times KM)$ variance covariance matrix of 
\footnotesize
\begin{align*}
\left(\widetilde{\phi}_{1, 1}(x, k, y; f, \epol)\ \widetilde{\phi}_{1, 2}(x, k, y; f, \epol)\ \cdots\ \widetilde{\phi}_{K, M}(x, k, y; f, \epol)\right)^\top.
\end{align*}
\normalsize

Then, for $\widetilde{D}$, let us define the estimator as 
\begin{align*}
\hat{\theta}^{\mathrm{BA2IPWIS}}_T = \zeta^\top_T \widetilde{D},
\end{align*}
where $\zeta_T$ is a data-dependent $KM$-dimensional vector such that $\sum^{M}_{\tau=1}\zeta^{1,\tau}_T = 1,\ \sum^{M}_{\tau=1}\zeta^{2,\tau}_T = 1,\ \cdots\ \sum^{M}_{\tau=1}\zeta^{K,\tau}_T = 1$. As well as Theorem~\ref{thm:main}, we have the following corollary.
\begin{corollary}
Under the same assumptions of Theorem~\ref{thm:main},
\begin{align*}
\sqrt{N} (\zeta^\top_T \widetilde{D} - \theta_0)\xrightarrow{\mathrm{d}} \mathcal{N}\left(0, \zeta^\top_T\widetilde{\Sigma}\zeta_T\right),
\end{align*}
\end{corollary}
\begin{proof}
From the constraint of $\zeta_T$, we have
\begin{align*}
\zeta^\top \ddot{h}_T = \sum^K_{a=1}\sum^M_{\tau=1}\zeta_{a, \tau} g_{a, \tau, T} - \theta_0.
\end{align*}
Then,
\begin{align*}
\sqrt{N}\zeta^\top G = \sqrt{N}\left( \sum^K_{a=1}\sum^M_{\tau=1}\zeta_{a, \tau} z_{a, \tau} - \theta_0\right) \xrightarrow{d} \mathcal{N}\left(0, \zeta^\top \Sigma \zeta\right).
\end{align*}
\end{proof}

Next, we consider an efficient weight that minimizes $\zeta^\top_T\widetilde{\Sigma}\zeta_T$. The efficient weight $\zeta^*$ can be defined as the solution of the following constraint optimization problem:
\begin{align*}
\min_{\zeta\in\mathbb{R}^{KM}}&\ \zeta^\top \Sigma \zeta\\
\mathrm{s.t.}&\ \sum^{M}_{\tau=1}\zeta^{1,\tau} = 1,\ \sum^{M}_{\tau=1}\zeta^{2,\tau} = 1,\ \cdots\ \sum^{M}_{\tau=1}\zeta^{K,\tau} = 1.
\end{align*}

When we put Assumption~\ref{asm:overlap_pol2} instead of Assumption~\ref{asm:overlap_pol}, we can construct an estimator by removing an element $g_{a, \tau, T}$ such that $\pi_{\tau}(a\mid x, \Omega_{t_{\tau - 1}})=0$ from $\widetilde{D}_T$.

\section{Proof of Theorem~\ref{thm:badr_asymp}}
\label{appdx:badr_dist}

\begin{proof}
Because the BADR estimator can be decomposed as 
\begin{align*}
&\widehat{R}^{\mathrm{BADR}}_T(\epol)=\sum^M_{\tau=1}w_{T,\tau}\frac{1}{t_{\tau} - t_{\tau-1}}\sum^{t_\tau}_{t=t_{\tau-1}+1}\phi_t(X_t, A_t, Y_t; 1, \hat{f}_{t-1}, \hat{g}_{t-1}, \epol)\\
&=\sum^M_{\tau=1}w_{T,\tau}\frac{1}{t_{\tau} - t_{\tau-1}}\sum^{t_\tau}_{t=t_{\tau-1}+1}\phi_t(X_t, A_t, Y_t; 1, \hat{f}_{t-1}, \hat{g}_{t-1}, \epol)\\
&\ \ \ - \sum^M_{\tau=1}w_{T,\tau}\frac{1}{t_{\tau} - t_{\tau-1}}\sum^{t_\tau}_{t=t_{\tau-1}+1}\phi_t(X_t, A_t, Y_t; 1, \hat{f}_{t-1}, \pi_\tau, \epol)\\
&\ \ \ + \sum^M_{\tau=1}w_{T,\tau}\frac{1}{t_{\tau} - t_{\tau-1}}\sum^{t_\tau}_{t=t_{\tau-1}+1}\phi_t(X_t, A_t, Y_t; 1, \hat{f}_{t-1}, \pi_\tau, \epol).
\end{align*}
The last term $\sum^M_{\tau=1}w_{T,\tau}\frac{1}{t_{\tau} - t_{\tau-1}}\sum^{t_\tau}_{t=t_{\tau-1}+1}\phi_t(X_t, A_t, Y_t; 1, \hat{f}_{t-1}, \pi_\tau, \epol)$ is a special case of the BA2IPW estimator and has the asymptotic normality if $w_T = (w_{T,1}\ \cdots\ w_{T,M})^\top \xrightarrow{\mathrm{p}} w = (w_{1}\ \cdots\ w_{M})^\top$, $w_{T,\tau} > 0$ and $\sum^M_{\tau=1}w_{T,\tau} = 1$, and Assumptions~\ref{asm:overlap_pol} and \ref{asm:overlap_outcome} hold. Then, we want to show that

\begin{align*}
&\sum^M_{\tau=1}w_{T,\tau}\frac{1}{t_{\tau} - t_{\tau-1}}\sum^{t_\tau}_{t=t_{\tau-1}+1}\phi_t(X_t, A_t, Y_t; 1, \hat{f}_{t-1}, \hat{g}_{t-1}, \epol)\\
&\ \ \ - \sum^M_{\tau=1}w_{T,\tau}\frac{1}{t_{\tau} - t_{\tau-1}}\sum^{t_\tau}_{t=t_{\tau-1}+1}\phi_t(X_t, A_t, Y_t; 1, \hat{f}_{t-1}, \pi_\tau, \epol)\\
&= \mathrm{o}_p(1/\sqrt{T}).
\end{align*}

Because the $M$ is fixed, we only have to consider 
\begin{align*}
&\frac{1}{t_{\tau} - t_{\tau-1}}\sum^{t_\tau}_{t=t_{\tau-1}+1}\phi_t(X_t, A_t, Y_t; 1, \hat{f}_{t-1}, \hat{g}_{t-1}, \epol) - \frac{1}{t_{\tau} - t_{\tau-1}}\sum^{t_\tau}_{t=t_{\tau-1}+1}\phi_t(X_t, A_t, Y_t; 1, \hat{f}_{t-1}, \pi_\tau, \epol)\\
&=\frac{1}{r_\tau T}\sum^{t_\tau}_{t=t_{\tau-1}+1}\phi_t(X_t, A_t, Y_t; 1, \hat{f}_{t-1}, \hat{g}_{t-1}, \epol) - \frac{1}{r_1T}\sum^{t_\tau}_{t=t_{\tau-1}+1}\phi_t(X_t, A_t, Y_t; 1, \hat{f}_{t-1}, \pi_\tau, \epol)\\
&= \mathrm{o}_p(1/\sqrt{T}).
\end{align*}

Here, for 
\begin{align*}
&\psi_1(X_t, A_t, Y_t; f, g)=\sum^K_{a=1}\frac{\epol(a\mid X_t)\mathbbm{1}[A_t=a]\left(Y_t - f(a, X_t)\right) }{g(a\mid X_t)},\\
&\psi_2(X_t; f)=\sum^K_{a=1}\epol(a\mid X_t)f(a, X_t),
\end{align*}
we have
\begin{align*}
&\frac{1}{r_\tau T}\sum^{t_\tau}_{t=t_{\tau-1}+1}\psi_t(X_t, A_t, Y_t; 1, \hat{f}_{t-1}, \hat{g}_{t-1}, \epol) - \frac{1}{r_\tau T}\sum^{t_\tau}_{t=t_{\tau-1}+1}\psi_t(X_t, A_t, Y_t; 1, \hat{f}_{t-1}, \pi_\tau, \epol)\\
&=\frac{1}{r_\tau T}\sum^{t_{\tau}}_{t=t_{\tau-1}+1}\Bigg\{\psi_1(X_t, A_t, Y_t; \hat{g}_{t-1}, \hat{f}_{t-1}) - \psi_1(X_t, A_t, Y_t; \pi_{t-1}, f^*)\\
&\ \ \ -\mathbb{E}\left[\psi_1(X_t, A_t, Y_t; \hat{g}_{t-1}, \hat{f}_{t-1}) - \psi_1(X_t, A_t, Y_t; \pi_{t-1}, f^*)\mid \Omega_{t-1}\right]\\
&\ \ \ + \psi_2(X_t; \hat{f}_{t-1}) - \psi_2(X_t; f^*) -\mathbb{E}\left[\psi_2(X_t; \hat{f}_{t-1}) - \psi_2(X_t; f^*)\mid \Omega_{t-1}\right]\Bigg\}\\
&\ \ \ + \frac{1}{r_\tau T}\sum^{t_{\tau}}_{t=t_{\tau-1}+1}\mathbb{E}\left[\psi_1(X_t, A_t, Y_t; \hat{g}_{t-1}, \hat{f}_{t-1})\mid \Omega_{t-1}\right] + \frac{1}{r_\tau T}\sum^{t_{\tau}}_{t=t_{\tau-1}+1}\mathbb{E}\left[\psi_2(X_t; \hat{f}_{t-1})\mid \Omega_{t-1}\right]\\
&\ \ \ - \frac{1}{r_\tau T}\sum^{t_{\tau}}_{t=t_{\tau-1}+1}\mathbb{E}\left[\psi_1(X_t, A_t, Y_t; \pi_{t-1}, f^*)\mid \Omega_{t-1}\right] - \frac{1}{r_\tau T}\sum^{t_{\tau}}_{t=t_{\tau-1}+1}\mathbb{E}\left[\psi_2(X_t; f^*)\mid \Omega_{t-1}\right].
\end{align*}
In the following parts, we separately show that
\begin{align}
\label{eq:part1}
&\sqrt{T}\frac{1}{r_\tau T}\sum^{t_{\tau}}_{t=t_{\tau-1}+1}\Bigg\{\psi_1(X_t, A_t, Y_t; \hat{g}_{t-1}, \hat{f}_{t-1}) - \psi_1(X_t, A_t, Y_t; \pi_{t-1}, f^*)\\
&\ \ \ -\mathbb{E}\left[\psi_1(X_t, A_t, Y_t; \hat{g}_{t-1}, \hat{f}_{t-1}) - \psi_1(X_t, A_t, Y_t; \pi_{t-1}, f^*)\mid \Omega_{t-1}\right]\nonumber\\
&\ \ \ + \psi_2(X_t; \hat{f}_{t-1}) - \psi_2(X_t; f^*) -\mathbb{E}\left[\phi_2(X_t; \hat{f}_{t-1}) - \phi_2(X_t; f^*)\mid \Omega_{t-1}\right]\Bigg\}\nonumber\\
&= \mathrm{o}_p(1);\nonumber
\end{align}
and 
\begin{align}
\label{eq:part2}
&\frac{1}{r_\tau T}\sum^{t_{\tau}}_{t=t_{\tau-1}+1}\mathbb{E}\left[\psi_1(X_t, A_t, Y_t; \hat{g}_{t-1}, \hat{f}_{t-1})\mid \Omega_{t-1}\right] + \frac{1}{r_\tau T}\sum^{t_{\tau}}_{t=t_{\tau-1}+1}\mathbb{E}\left[\psi_2(X_t; \hat{f}_{t-1})\mid \Omega_{t-1}\right]\\
&- \frac{1}{r_\tau T}\sum^{t_{\tau}}_{t=t_{\tau-1}+1}\mathbb{E}\left[\psi_1(X_t, A_t, Y_t; \pi_{t-1}, f^*)\mid \Omega_{t-1}\right] - \frac{1}{r_\tau T}\sum^{t_{\tau}}_{t=t_{\tau-1}+1}\mathbb{E}\left[\psi_2(X_t; f^*)\mid \Omega_{t-1}\right] = \mathrm{o}_p(1/\sqrt{T}).\nonumber
\end{align}

\subsection{Proof of \eqref{eq:part1}}

For any $\varepsilon > 0$, to show that 
\begin{align*}
&\mathbb{P}\Bigg(\Bigg|\sqrt{T}\frac{1}{r_\tau T}\sum^{t_{\tau}}_{t=t_{\tau-1}+1}\Bigg\{\psi_1(X_t, A_t, Y_t; \hat{g}_{t-1}, \hat{f}_{t-1}) - \psi_1(X_t, A_t, Y_t; \pi_{t-1}, f^*)\\
&\ \ \ -\mathbb{E}\left[\psi_1(X_t, A_t, Y_t; \hat{g}_{t-1}, \hat{f}_{t-1}) - \psi_1(X_t, A_t, Y_t; \pi_{t-1}, f^*)\mid \Omega_{t-1}\right]\\
&\ \ \ + \psi_2(X_t; \hat{f}_{t-1}) - \psi_2(X_tt; f^*) -\mathbb{E}\left[\psi_2(X_t; \hat{f}_{t-1}) - \psi_2(X_t; f^*)\mid \Omega_{t-1}\right]\Bigg\}\Bigg| > \varepsilon \Bigg)\\
& \to 0,
\end{align*}
we show that the mean is $0$ and the variance of the component converges to $0$. Then, from the Chebyshev's inequality, this result yields the statement.

The mean is calculated as 
\begin{align*}
&\sqrt{T}\frac{1}{r_\tau T}\sum^{t_{\tau}}_{t=t_{\tau-1}+1}\mathbb{E}\Bigg[\Bigg\{\psi_1(X_t, A_t, Y_t; \hat{g}_{t-1}, \hat{f}_{t-1}) - \psi_1(X_t, A_t, Y_t; \pi_{t-1}, f^*)\\
&\ \ \ \ \ \ \ -\mathbb{E}\left[\psi_1(X_t, A_t, Y_t; \hat{g}_{t-1}, \hat{f}_{t-1}) - \psi_1(X_t, A_t, Y_t; \pi_{t-1}, f^*)\mid \Omega_{t-1}\right]\\
&\ \ \ \ \ \ \ + \psi_2(X_t; \hat{f}_{t-1}) - \psi_2(X_t; f^*) -\mathbb{E}\left[\psi_2(X_t; \hat{f}_{t-1}) - \psi_2(X_t; f^*)\mid \Omega_{t-1}\right]\Bigg\}\Bigg]\\
&=\sqrt{T}\frac{1}{r_\tau T}\sum^{t_{\tau}}_{t=t_{\tau-1}+1}\mathbb{E}\Bigg[\mathbb{E}\Bigg[\Bigg\{\psi_1(X_t, A_t, Y_t; \hat{g}_{t-1}, \hat{f}_{t-1}) - \psi_1(X_t, A_t, Y_t; \pi_{t-1}, f^*)\\
&\ \ \ \ \ \ \ -\mathbb{E}\left[\psi_1(X_t, A_t, Y_t; \hat{g}_{t-1}, \hat{f}_{t-1}) - \psi_1(X_t, A_t, Y_t; \pi_{t-1}, f^*)\mid \Omega_{t-1}\right]\\
&\ \ \ \ \ \ \ + \psi_2(X_t; \hat{f}_{t-1}) - \psi_2(X_t; f^*) -\mathbb{E}\left[\psi_2(X_t; \hat{f}_{t-1}) - \psi_2(X_t; f^*)\mid \Omega_{t-1}\right]\Bigg\}\mid \Omega_{t-1}\Bigg]\Bigg]\\
& = 0
\end{align*}

Because the mean is $0$, the variance is calculated as 
\begin{align*}
&\mathrm{Var}\Bigg(\sqrt{T}\frac{1}{r_\tau T}\sum^{t_{\tau}}_{t=t_{\tau-1}+1}\Bigg\{\psi_1(X_t, A_t, Y_t; \hat{g}_{t-1}, \hat{f}_{t-1}) - \psi_1(X_t, A_t, Y_t; \pi_{t-1}, f^*)\\
&\ \ \ \ \ \ \ -\mathbb{E}\left[\psi_1(X_t, A_t, Y_t; \hat{g}_{t-1}, \hat{f}_{t-1}) - \psi_1(X_t, A_t, Y_t; \pi_{t-1}, f^*)\mid \Omega_{t-1}\right]\\
&\ \ \ \ \ \ \ + \psi_2(X_t; \hat{f}_{t-1}) - \psi_2(X_t; f^*) -\mathbb{E}\left[\psi_2(X_t; \hat{f}_{t-1}) - \psi_2(X_t; f^*)\mid \Omega_{t-1}\right]\Bigg\}\Bigg)\\
&=\frac{1}{r^2_\tau T}\mathbb{E}\Bigg[\Bigg(\sum^{t_{\tau}}_{t=t_{\tau-1}+1}\Bigg\{\psi_1(X_t, A_t, Y_t; \hat{g}_{t-1}, \hat{f}_{t-1}) - \psi_1(X_t, A_t, Y_t; \pi_{t-1}, f^*)\\
&\ \ \ \ \ \ \ -\mathbb{E}\left[\psi_1(X_t, A_t, Y_t; \hat{g}_{t-1}, \hat{f}_{t-1}) - \psi_1(X_t, A_t, Y_t; \pi_{t-1}, f^*)\mid \Omega_{t-1}\right]\\
&\ \ \ \ \ \ \ + \psi_2(X_t; \hat{f}_{t-1}) - \psi_2(X_t; f^*) -\mathbb{E}\left[\psi_2(X_t; \hat{f}_{t-1}) - \psi_2(X_t; f^*)\mid \Omega_{t-1}\right]\Bigg\}\Bigg)^2\Bigg].
\end{align*}
Then, we can decompose the variance as
\begin{align*}
&=\frac{1}{r^2_\tau T}\sum^{t_{\tau}}_{t=t_{\tau-1}+1}\mathbb{E}\Bigg[\Bigg(\psi_1(X_t, A_t, Y_t; \hat{g}_{t-1}, \hat{f}_{t-1}) - \psi_1(X_t, A_t, Y_t; \pi_{t-1}, f^*)\\
&\ \ \ \ \ \ \ -\mathbb{E}\left[\psi_1(X_t, A_t, Y_t; \hat{g}_{t-1}, \hat{f}_{t-1}) - \psi_1(X_t, A_t, Y_t; \pi_{t-1}, f^*)\mid \Omega_{t-1}\right]\\
&\ \ \ \ \ \ \ + \psi_2(X_t; \hat{f}_{t-1}) - \psi_2(X_t; f^*) -\mathbb{E}\left[\psi_2(X_t; \hat{f}_{t-1}) - \psi_2(X_t; f^*)\mid \Omega_{t-1}\right]\Bigg)^2\Bigg]\\
&\ \ \ +\frac{2}{r^2_\tau T}\sum^{t_{\tau}-1}_{t=1}\sum^{t_{\tau}}_{s=t+1}\mathbb{E}\Bigg[\Bigg(\psi_1(X_t, A_t, Y_t; \hat{g}_{t-1}, \hat{f}_{t-1}) - \psi_1(X_t, A_t, Y_t; \pi_{t-1}, f^*)\\
&\ \ \ \ \ \ \ -\mathbb{E}\left[\psi_1(X_t, A_t, Y_t; \hat{g}_{t-1}, \hat{f}_{t-1}) - \psi_1(X_t, A_t, Y_t; \pi_{t-1}, f^*)\mid \Omega_{t-1}\right]\\
&\ \ \ \ \ \ \ + \psi_2(X_t; \hat{f}_{t-1}) - \psi_2(X_t; f^*) -\mathbb{E}\left[\psi_2(X_t; \hat{f}_{t-1}) - \psi_2(X_t; f^*)\mid \Omega_{t-1}\right]\Bigg)\\
&\ \ \ \ \ \ \ \times \Bigg(\psi_1(X_s, A_s, Y_s; \hat{g}_{s-1}, \hat{f}_{s-1}) - \psi_1(X_s, A_s, Y_s; \pi_{s-1}, f^*)\\
&\ \ \ \ \ \ \ -\mathbb{E}\left[\psi_1(X_s, A_s, Y_s; \hat{g}_{s-1}, \hat{f}_{s-1}) - \psi_1(X_s, A_s, Y_s; \pi_{s-1}, f^*)\mid \Omega_{s-1}\right]\\
&\ \ \ \ \ \ \ + \psi_2(X_s; \hat{f}_{s-1}) - \psi_2(X_s; f^*) -\mathbb{E}\left[\psi_2(X_s; \hat{f}_{s-1}) - \psi_2(X_s; f^*)\mid \Omega_{s-1}\right]\Bigg)\Bigg].
\end{align*}

For $s > t$, we can vanish the covariance terms as 
\begin{align*}
&\mathbb{E}\Bigg[\Bigg(\psi_1(X_t, A_t, Y_t; \hat{g}_{t-1}, \hat{f}_{t-1}) - \psi_1(X_t, A_t, Y_t; \pi_{t-1}, f^*)\\
&\ \ \ \ \ \ \ -\mathbb{E}\left[\psi_1(X_t, A_t, Y_t; \hat{g}_{t-1}, \hat{f}_{t-1}) - \psi_1(X_t, A_t, Y_t; \pi_{t-1}, f^*)\mid \Omega_{t-1}\right]\\
&\ \ \ \ \ \ \ + \psi_2(X_t; \hat{f}_{t-1}) - \psi_2(X_t; f^*) -\mathbb{E}\left[\psi_2(X_t; \hat{f}_{t-1}) - \psi_2(X_t; f^*)\mid \Omega_{t-1}\right]\Bigg)\\
&\ \ \ \ \ \ \ \times \Bigg(\psi_1(X_s, A_s, Y_s; \hat{g}_{s-1}, \hat{f}_{s-1}) - \psi_1(X_s, A_s, Y_s; \pi_{s-1}, f^*)\\
&\ \ \ \ \ \ \ -\mathbb{E}\left[\psi_1(X_s, A_s, Y_s; \hat{g}_{s-1}, \hat{f}_{s-1}) - \psi_1(X_s, A_s, Y_s; \pi_{s-1}, f^*)\mid \Omega_{s-1}\right]\\
&\ \ \ \ \ \ \ + \psi_2(X_s; \hat{f}_{s-1}) - \psi_2(X_s; f^*) -\mathbb{E}\left[\psi_2(X_s; \hat{f}_{s-1}) - \psi_2(X_s; f^*)\mid \Omega_{s-1}\right]\Bigg)\Bigg]\\
&=\mathbb{E}\Bigg[U\mathbb{E}\Bigg[\Bigg(\psi_1(X_s, A_s, Y_s; \hat{g}_{s-1}, \hat{f}_{s-1}) - \psi_1(X_s, A_s, Y_s; \pi_{s-1}, f^*)\\
&\ \ \ \ \ \ \ -\mathbb{E}\left[\psi_1(X_s, A_s, Y_s; \hat{g}_{s-1}, \hat{f}_{s-1}) - \psi_1(X_s, A_s, Y_s; \pi_{s-1}, f^*)\mid \Omega_{s-1}\right]\\
&\ \ \ \ \ \ \ + \psi_2(X_s; \hat{f}_{s-1}) - \psi_2(X_s; f^*) -\mathbb{E}\left[\psi_2(X_s; \hat{f}_{s-1}) - \psi_2(X_s; f^*)\mid \Omega_{s-1}\right]\Bigg)\mid \Omega_{s-1}\Bigg]\Bigg]\\
&=0,
\end{align*}
where $U=\Bigg(\psi_1(X_t, A_t, Y_t; \hat{g}_{t-1}, \hat{f}_{t-1}) - \psi_1(X_t, A_t, Y_t; \pi_{t-1}, f^*)-\mathbb{E}\left[\psi_1(X_t, A_t, Y_t; \hat{g}_{t-1}, \hat{f}_{t-1}) - \psi_1(X_t, A_t, Y_t; \pi_{t-1}, f^*)\mid \Omega_{t-1}\right] + \psi_2(X_t; \hat{f}_{t-1}) - \psi_2(X_t; f^*) -\mathbb{E}\left[\psi_2(X_t; \hat{f}_{t-1}) - \psi_2(X_t; f^*)\mid \Omega_{t-1}\right]\Bigg)$. Therefore, the variance is calculated as 
\begin{align*}
&\mathrm{Var}\Bigg(\sqrt{T}\frac{1}{r_\tau T}\sum^{t_{\tau}}_{t=t_{\tau-1}+1}\Bigg\{\psi_1(X_t, A_t, Y_t; \hat{g}_{t-1}, \hat{f}_{t-1}) - \psi_1(X_t, A_t, Y_t; \pi_{t-1}, f^*)\\
&\ \ \ \ \ \ \ -\mathbb{E}\left[\psi_1(X_t, A_t, Y_t; \hat{g}_{t-1}, \hat{f}_{t-1}) - \psi_1(X_t, A_t, Y_t; \pi_{t-1}, f^*)\mid \Omega_{t-1}\right]\\
&\ \ \ \ \ \ \ + \psi_2(X_t; \hat{f}_{t-1}) - \psi_2(X_t; f^*)\\
&\ \ \ \ \ \ \ -\mathbb{E}\left[\psi_2(X_t; \hat{f}_{t-1}) - \psi_2(X_t; f^*)\mid \Omega_{t-1}\right]\Bigg\}\Bigg)\\
&=\frac{1}{r^2_\tau T}\sum^{t_{\tau}}_{t=t_{\tau-1}+1}\mathbb{E}\Bigg[\Bigg(\psi_1(X_t, A_t, Y_t; \hat{g}_{t-1}, \hat{f}_{t-1}) - \psi_1(X_t, A_t, Y_t; \pi_{t-1}, f^*)\\
&\ \ \ \ \ \ \ -\mathbb{E}\left[\psi_1(X_t, A_t, Y_t; \hat{g}_{t-1}, \hat{f}_{t-1}) - \psi_1(X_t, A_t, Y_t; \pi_{t-1}, f^*)\mid \Omega_{t-1}\right]\\
&\ \ \ \ \ \ \ + \psi_2(X_t; \hat{f}_{t-1}) - \psi_2(X_t; f^*)\\
&\ \ \ \ \ \ \ -\mathbb{E}\left[\psi_2(X_t; \hat{f}_{t-1}) - \psi_2(X_t; f^*)\mid \Omega_{t-1}\right]\Bigg)^2\Bigg].
\end{align*}

Then, by considering the conditional expectation,
\begin{align*}
&=\frac{1}{r^2_\tau T}\sum^{t_{\tau}}_{t=t_{\tau-1}+1}\mathbb{E}\Bigg[\mathbb{E}\Bigg[\Bigg(\psi_1(X_t, A_t, Y_t; \hat{g}_{t-1}, \hat{f}_{t-1}) - \psi_1(X_t, A_t, Y_t; \pi_{t-1}, f^*)\\
&\ \ \ \ \ \ \ -\mathbb{E}\left[\psi_1(X_t, A_t, Y_t; \hat{g}_{t-1}, \hat{f}_{t-1}) - \psi_1(X_t, A_t, Y_t; \pi_{t-1}, f^*)\mid \Omega_{t-1}\right]\\
&\ \ \ \ \ \ \ + \psi_2(X_t; \hat{f}_{t-1}) - \psi_2(X_t; f^*)\\
&\ \ \ \ \ \ \ -\mathbb{E}\left[\psi_2(X_t; \hat{f}_{t-1}) - \psi_2(X_t; f^*)\mid \Omega_{t-1}\right]\Bigg)^2\mid \Omega_{t-1}\Bigg]\Bigg]\\
&=\frac{1}{r^2_\tau T}\sum^{t_{\tau}}_{t=t_{\tau-1}+1}\mathbb{E}\Bigg[\mathbb{E}\Bigg[\Bigg(\psi_1(X_t, A_t, Y_t; \hat{g}_{t-1}, \hat{f}_{t-1}) - \psi_1(X_t, A_t, Y_t; \pi_{t-1}, f^*)\\
&\ \ \ \ \ \ \ + \psi_2(X_t; \hat{f}_{t-1}) - \psi_2(X_t; f^*)\Bigg)^2\mid \Omega_{t-1}\Bigg]\\
&\ \ \ \ \ \ \ -\Bigg(\mathbb{E}\left[\psi_1(X_t, A_t, Y_t; \hat{g}_{t-1}, \hat{f}_{t-1}) - \psi_1(X_t, A_t, Y_t; \pi_{t-1}, f^*)\mid \Omega_{t-1}\right]\\
&\ \ \ \ \ \ \ +\mathbb{E}\left[\psi_2(X_t; \hat{f}_{t-1}) - \psi_2(X_t; f^*)\mid \Omega_{t-1}\right]\Bigg)^2\Bigg]\\
&\leq \frac{1}{r^2_\tau T}\sum^{t_{\tau}}_{t=t_{\tau-1}+1}\mathbb{E}\Bigg[\Bigg|\mathbb{E}\Bigg[\Bigg(\psi_1(X_t, A_t, Y_t; \hat{g}_{t-1}, \hat{f}_{t-1}) - \psi_1(X_t, A_t, Y_t; \pi_{t-1}, f^*)\\
&\ \ \ \ \ \ \ + \psi_2(X_t; \hat{f}_{t-1}) - \psi_2(X_t; f^*)\Bigg)^2\mid \Omega_{t-1}\Bigg]\Bigg|\Bigg].
\end{align*}

Then, we want to show
\begin{align*}
&\mathbb{E}\Bigg[\Bigg|\mathbb{E}\Bigg[\Bigg(\psi_1(X_t, A_t, Y_t; \hat{g}_{t-1}, \hat{f}_{t-1}) - \psi_1(X_t, A_t, Y_t; \pi_{t-1}, f^*) + \psi_2(X_t; \hat{f}_{t-1}) - \psi_2(X_t; f^*)\Bigg)^2\mid \Omega_{t-1}\Bigg]\Bigg|\Bigg]\\
&\to 0.
\end{align*}

Here, we can use
\begin{align}
\label{eq:first_e}
&\mathbb{E}\left[\left\{\sum^K_{a=1}\frac{\epol(a\mid X_t)\mathbbm{1}[A_t=a]\left(Y_t - \hat{f}_{t-1}(a, X_t)\right) }{\hat{g}_{t-1}(a\mid X_t)} - \sum^K_{a=1}\frac{\epol(a\mid X_t)\mathbbm{1}[A_t=a]\left(Y_t - f^*(a, X_t)\right) }{\pi_{t-1}(a\mid X_t, \Omega_{t-1})}\right\}^2\mid \Omega_{t-1}\right]\nonumber\\
& = \op(1),
\end{align}
and 
\begin{align}
\label{eq:second_e}
&\mathbb{E}\left[\left\{\sum^K_{a=1}\epol(a\mid X_t)\hat{f}_{t-1}(a, X_t) - \sum^K_{a=1}\epol(a\mid X_t)f^*(a, X_t)\right\}^2\mid \Omega_{t-1}\right]\nonumber\\
& = \op(1).
\end{align}

The first equation \eqref{eq:first_e} is proved by 

\begin{align*}
&\mathbb{E}\left[\left\{\sum^K_{a=1}\frac{\epol(a\mid X_t)\mathbbm{1}[A_t=a]\left(Y_t - \hat{f}_{t-1}(a, X_t)\right) }{\hat{g}_{t-1}(a\mid X_t)} - \sum^K_{a=1}\frac{\epol(a\mid X_t)\mathbbm{1}[A_t=a]\left(Y_t - f^*(a, X_t)\right) }{\pi_{t-1}(a\mid X_t, \Omega_{t-1})}\right\}^2\mid \Omega_{t-1}\right]\\
&=\mathbb{E}\Bigg[\Bigg\{\sum^K_{a=1}\frac{\epol(a\mid X_t)\mathbbm{1}[A_t=a]\left(Y_t - \hat{f}_{t-1}(a, X_t)\right) }{\hat{g}_{t-1}(a\mid X_t)} - \sum^K_{a=1}\frac{\epol(a\mid X_t)\mathbbm{1}[A_t=a]\left(Y_t - f^*(a, X_t)\right) }{\hat{g}_{t-1}(a\mid X_t)}\\
&\ \ \ + \sum^K_{a=1}\frac{\epol(a\mid X_t)\mathbbm{1}[A_t=a]\left(Y_t - f^*(a, X_t)\right) }{\hat{g}_{t-1}(a\mid X_t)}- \sum^K_{a=1}\frac{\epol(a\mid X_t)\mathbbm{1}[A_t=a]\left(Y_t - f^*(a, X_t)\right) }{\pi_{t-1}(a\mid X_t, \Omega_{t-1})}\Bigg\}^2\mid \Omega_{t-1}\Bigg]\\
&\leq 2\mathbb{E}\Bigg[\Bigg\{\sum^K_{a=1}\frac{\epol(a\mid X_t)\mathbbm{1}[A_t=a]\left(Y_t - \hat{f}_{t-1}(a, X_t)\right) }{\hat{g}_{t-1}(a\mid X_t)}- \sum^K_{a=1}\frac{\epol(a\mid X_t)\mathbbm{1}[A_t=a]\left(Y_t - f^*(a, X_t)\right) }{\hat{g}_{t-1}(a\mid X_t)}\mid \Omega_{t-1}\Bigg]\\
&\ \ \ + 2\mathbb{E}\Bigg[\Bigg\{\sum^K_{a=1}\frac{\epol(a\mid X_t)\mathbbm{1}[A_t=a]\left(Y_t - f^*(a, X_t)\right) }{\hat{g}_{t-1}(a\mid X_t)} - \sum^K_{a=1}\frac{\epol(a\mid X_t)\mathbbm{1}[A_t=a]\left(Y_t - f^*(a, X_t)\right) }{\pi_{t-1}(a\mid X_t, \Omega_{t-1})}\Bigg\}^2\mid \Omega_{t-1}\Bigg]\\
&\leq 2C\|f^* - \hat{f}_{t-1}\|^2_2 + 2\times 4C\|\hat{g}_{t-1} - \pi_{t-1}\|^2_{2} = \op(1),
\end{align*}
where $C>0$ is a constant. Here, we have used a parallelogram law from the second line to the third line. We have use $|\hat{f}_{t-1}| < C_3$, and $0<\frac{\epol}{\hat{g}_t} < C_4$ and convergence rate conditions, from the third line to the fourth line.  The second equation \eqref{eq:second_e} is proved by Jensen's inequality. 

Besides, we can also use 
\begin{align}
\label{eq:third_e}
&\mathbb{E}\Bigg[\left\{\sum^K_{a=1}\frac{\epol(a\mid X_t)\mathbbm{1}[A_t=a]\left(Y_t - \hat{f}_{t-1}(a, X_t)\right) }{\hat{g}_{t-1}(a\mid X_t)} - \sum^K_{a=1}\frac{\epol(a\mid X_t)\mathbbm{1}[A_t=a]\left(Y_t - f^*(a, X_t)\right) }{\pi_{t-1}(a\mid X_t, \Omega_{t-1})}\right\}\nonumber\\
&\ \ \ \left\{\sum^K_{a=1}\epol(a\mid X_t)\hat{f}_{t-1}(a, X_t) - \sum^K_{a=1}\epol(a\mid X_t)f^*(a, X_t)\right\}\mid \Omega_{t-1}\Bigg]\nonumber\\
& = \op(1)
\end{align}
This is proved by
\begin{align*}
&\mathbb{E}\Bigg[\left\{\sum^K_{a=1}\frac{\epol(a\mid X_t)\mathbbm{1}[A_t=a]\left(Y_t - \hat{f}_{t-1}(a, X_t)\right) }{\hat{g}_{t-1}(a\mid X_t)} - \sum^K_{a=1}\frac{\epol(a\mid X_t)\mathbbm{1}[A_t=a]\left(Y_t - f^*(a, X_t)\right) }{\pi_{t-1}(a\mid X_t, \Omega_{t-1})}\right\}\\
&\ \ \ \left\{\sum^K_{a=1}\epol(a\mid X_t)\hat{f}_{t-1}(a, X_t) - \sum^K_{a=1}\epol(a\mid X_t)f^*(a, X_t)\right\}\mid \Omega_{t-1}\Bigg]\\
&\leq C\Bigg|\mathbb{E}\Bigg[\Big\{\sum^K_{a=1}\left(\hat{g}_{t-1}(a\mid X_t) - \pi_{t-1}(a\mid X_t, \Omega_{t-1})\right)\Big\}\\
&\ \ \ \times \Big\{\sum^K_{a=1}\left(\epol(a\mid X_t)\hat{f}_{t-1}(a, X_t) - \epol(a\mid X_t)f^*(a, X_t)\right)\Big\}\mid \Omega_{t-1}\Bigg]\Bigg|\\
& = \op(1),
\end{align*}
where $C>0$ is a constant. Here, we used \Holder's inequality $\|fg \|_1 \leq  \|f \|_2  \|g \|_2$ and 
\begin{align*}
& \leq C \Bigg\| \sum^K_{a=1}\Big(\pi_{t-1}(a\mid X_t, \Omega_{t-1}) - \hat{g}_{t-1}(a\mid X_t)\Big)\Bigg\|_{2}\Bigg\|\sum^K_{a=1}\Big(f^*(a, X_t) - \hat{f}_{t-1}(a, X_t) \Big)\Bigg\|_{2}\\
& = \op(1)
\end{align*}

Therefore, from the $L^r$ convergence theorem (Proposition~\ref{prp:lr_conv_theorem}) and the boundedness of the random variables, we can show that as $t\to \infty$,
\begin{align*}
&\mathbb{E}\Bigg[\Bigg|\mathbb{E}\Bigg[\Bigg(\psi_1(X_t, A_t, Y_t; \hat{g}_{t-1}, \hat{f}_{t-1}) - \psi_1(X_t, A_t, Y_t; \pi_{t-1}, f^*)\\
&\ \ \ \ \ \ \ + \psi_2(X_t, A_t, Y_t; \hat{f}_{t-1}) - \psi_2(X_t, A_t, Y_t; f^*)\Bigg)^2\mid \Omega_{t-1}\Bigg]\Bigg|\Bigg]\\
&\to 0.
\end{align*}

Therefore, for any $\epsilon > 0$, there exists a constant $\tilde C > 0$ such that 
\begin{align*}
&\frac{1}{r_\tau T} \sum^{t_\tau}_{t=t_{\tau-1}+1}\mathbb{E}\Bigg[\mathbb{E}\Bigg[\Bigg(\psi_1(X_t, A_t, Y_t; \hat{g}_{t-1}, \hat{f}_{t-1}) - \psi_1(X_t, A_t, Y_t; \pi_{t-1}, f^*)+ \psi_2(X_t; \hat{f}_{t-1}) - \psi_2(X_t; f^*)\Bigg)^2\mid \Omega_{t-1}\Bigg]\Bigg]\\
&\leq \tilde C/T + \epsilon.
\end{align*}
Thus, the variance also converges to $0$. Then, from Chebyshev’s inequality,
\begin{align*}
&\mathbb{P}\Bigg(\Bigg|\sqrt{T}\frac{1}{r_\tau T}\sum^{t_{\tau}}_{t=t_{\tau-1}+1}\Bigg\{\psi_1(X_t, A_t, Y_t; \hat{g}_{t-1}, \hat{f}_{t-1}) - \psi_1(X_t, A_t, Y_t; \pi_{t-1}, f^*)\\
&\ \ \ -\mathbb{E}\left[\psi_1(X_t, A_t, Y_t; \hat{g}_{t-1}, \hat{f}_{t-1}) - \psi_1(X_t, A_t, Y_t; \pi_{t-1}, f^*)\mid \Omega_{t-1}\right]\\
&\ \ \ + \psi_2(X_t; \hat{f}_{t-1}) - \psi_2(X_t; f^*) -\mathbb{E}\left[\psi_2(X_t; \hat{f}_{t-1}) - \psi_2(X_t; f^*)\mid \Omega_{t-1}\right]\Bigg\}\Bigg| > \varepsilon\Bigg)\\
&\leq \mathrm{Var}\Bigg(\sqrt{T}\frac{1}{r_\tau T}\sum^{t_{\tau}}_{t=t_{\tau-1}+1}\Bigg\{\psi_1(X_t, A_t, Y_t; \hat{g}_{t-1}, \hat{f}_{t-1}) - \psi_1(X_t, A_t, Y_t; \pi_{t-1}, f^*)\\
&\ \ \ -\mathbb{E}\left[\psi_1(X_t, A_t, Y_t; \hat{g}_{t-1}, \hat{f}_{t-1}) - \psi_1(X_t, A_t, Y_t; \pi_{t-1}, f^*)\mid \Omega_{t-1}\right]\\
&\ \ \ + \psi_2(X_t; \hat{f}_{t-1}) - \psi_2(X_t; f^*) -\mathbb{E}\left[\psi_2(X_t; \hat{f}_{t-1}) - \psi_2(X_t; f^*)\mid \Omega_{t-1}\right]\Bigg\}\Bigg)/\varepsilon^2\\
&\to 0.
\end{align*}

\subsection{Proof of \eqref{eq:part2}}

\begin{align}
&\frac{1}{r_\tau T}\sum^{t_{\tau}}_{t=t_{\tau-1}+1}\mathbb{E}\left[\psi_1(X_t, A_t, Y_t; \hat{g}_{t-1}, \hat{f}_{t-1})\mid \Omega_{t-1}\right] + \frac{1}{r_\tau T}\sum^{t_{\tau}}_{t=t_{\tau-1}+1}\mathbb{E}\left[\psi_2(X_t; \hat{f}_{t-1})\mid \Omega_{t-1}\right]\nonumber\\
&\ \ \ - \frac{1}{r_\tau T}\sum^{t_{\tau}}_{t=t_{\tau-1}+1}\mathbb{E}\left[\psi_1(X_t, A_t, Y_t; \pi_{t-1}, f^*)\mid \Omega_{t-1}\right] - \frac{1}{r_\tau T}\sum^{t_{\tau}}_{t=t_{\tau-1}+1}\mathbb{E}\left[\psi_2(X_t; f^*)\mid \Omega_{t-1}\right]\nonumber\\
& = \frac{1}{r_\tau T}\sum^{t_{\tau}}_{t=t_{\tau-1}+1} \mathbb{E}\left[\sum^K_{a=1}\frac{\epol(a\mid X_t)\mathbbm{1}[A_t=a]\left(Y_t - \hat{f}_{t-1}(a, X_t)\right) }{\hat{g}_{t-1}(a\mid X_t)}\mid \Omega_{t-1}\right]\nonumber\\
&\ \ \ + \frac{1}{r_\tau T}\sum^{t_{\tau}}_{t=t_{\tau-1}+1}\mathbb{E}\left[\sum^K_{a=1}\epol(a, X_t)\hat{f}_{t-1}(a, X_t)\mid \Omega_{t-1}\right]\nonumber\\
\label{eq:vanish1}
&\ \ \ - \frac{1}{r_\tau T}\sum^{t_{\tau}}_{t=t_{\tau-1}+1}\mathbb{E}\left[\sum^K_{a=1}\frac{\epol(a\mid X_t)\mathbbm{1}[A_t=a]\left(Y_t - f^*(a, X_t)\right) }{\pi_{t-1}(a\mid X_t, \Omega_{t-1})}\mid \Omega_{t-1}\right]\\
&\ \ \ - \frac{1}{r_\tau T}\sum^{t_{\tau}}_{t=t_{\tau-1}+1}\mathbb{E}\left[\sum^K_{a=1}\epol(a, X_t)f^*(a, X_t)\mid \Omega_{t-1}\right]\nonumber.
\end{align}
Because \eqref{eq:vanish1} is $0$,
\begin{align*}
& = \frac{1}{r_\tau T}\sum^{t_{\tau}}_{t=t_{\tau-1}+1} \mathbb{E}\left[\sum^K_{a=1}\frac{\epol(a\mid X_t)\mathbbm{1}[A_t=a]\left(Y_t - \hat{f}_{t-1}(a, X_t)\right) }{\hat{g}_{t-1}(a\mid X_t)}\mid \Omega_{t-1}\right]\\
&\ \ \ + \frac{1}{T}\sum^{t_{\tau}}_{t=t_{\tau-1}+1}\mathbb{E}\left[\sum^K_{a=1}\epol(a, X_t)\hat{f}_{t-1}(a, X_t)\mid \Omega_{t-1}\right] - \frac{1}{T}\sum^{t_{\tau}}_{t=t_{\tau-1}+1}\mathbb{E}\left[\sum^K_{a=1}\epol(a, X_t)f^*(a, X_t)\mid \Omega_{t-1}\right]\\
& = \frac{1}{r_\tau T}\sum^{t_{\tau}}_{t=t_{\tau-1}+1} \mathbb{E}\left[\sum^K_{a=1}\frac{\epol(a\mid X_t)\mathbbm{1}[A_t=a]\left(Y_t - \hat{f}_{t-1}(a, X_t)\right) }{\hat{g}_{t-1}(a\mid X_t)}\mid \Omega_{t-1}\right]\\
&\ \ \ - \frac{1}{r_\tau T}\sum^{t_{\tau}}_{t=t_{\tau-1}+1}\mathbb{E}\left[\sum^K_{a=1}\epol(a, X_t)\Big( f^*(a, X_t)- \hat{f}_{t-1}(a, X_t))\Big) \mid \Omega_{t-1}\right]\\
& = \frac{1}{r_\tau T}\sum^{t_{\tau}}_{t=t_{\tau-1}+1} \sum^K_{a=1}\mathbb{E}\Bigg[\mathbb{E}\Bigg[\frac{\epol(a\mid X_t)\pi_{t-1}(a\mid X_t, \Omega_{t-1})\left(f^*(a, X_t) - \hat{f}_{t-1}(a, X_t)\right) }{\hat{g}_{t-1}(a\mid X_t)}\\
&\ \ \ \ \ \ \ \ \ \ \ \ \ \ \ \ \ \ \ \ \ \ \ \  - \epol(a, X_t)\Big( f^*(a, X_t)- \hat{f}_{t-1}(a, X_t))\Big) \mid X_t, \Omega_{t-1}\Bigg]\mid \Omega_{t-1}\Bigg]\\
& \leq \frac{1}{r_\tau T}\sum^{t_{\tau}}_{t=t_{\tau-1}+1} \sum^K_{a=1}\Bigg|\mathbb{E}\Bigg[ \frac{\epol(a\mid X_t)\Big(\pi_{t-1}(a\mid X_t) - \hat{g}_{t-1}(a\mid X_t)\Big)\left(f^*(a, X_t) - \hat{f}_{t-1}(a, X_t)\right) }{\hat{g}_{t-1}(a\mid X_t)}  \mid \Omega_{t-1}\Bigg]\Bigg|. 
\end{align*}
By using \Holder's inequality $\|fg \|_1 \leq  \|f \|_2  \|g \|_2$, for a constant $C>0$, we have
\begin{align*}
& \leq \frac{C}{r_\tau T}\sum^{t_{\tau}}_{t=t_{\tau-1}+1} \Bigg\| \pi_{t-1}(a\mid X_t, \Omega_{t-1}) - \hat{g}_{t-1}(a\mid X_t)\Bigg\|_{2}\Bigg\|f^*(a, X_t) - \hat{f}_{t-1}(a, X_t) \Bigg\|_{2}\\
& = \frac{C}{r_\tau T}\sum^{t_{\tau}}_{t=t_{\tau-1}+1}\alpha\beta\\
& = \frac{C}{r_\tau T}\sum^{t_{\tau}}_{t=t_{\tau-1}+1}\op((t-t_{\tau-1})^{-1/2})\\
& = \op(1/\sqrt{T}).
\end{align*}

\end{proof}

\section{Experiments}

\subsection{Additional Experimental Results}
\label{appdx:det_exp}

In this section, we show the additional experimental results using different sample sizes, nonparametric estimator, and the numbers of batches. 

In Table~\ref{tbl:appdx:exp_table1},  we show the results of OPE with $1,000$ samples under $10$ batches using nonparametric NW regression. The upper graph shows the results with the RW policy. The lower graph shows the results with the UCB policy.

In Table~\ref{tbl:appdx:exp_table2},  we show the results of OPE with $2,000$ samples under $10$ batches using nonparametric NW regression. The upper graph shows the results with the RW policy. The lower graph shows the results with the UCB policy.

In Table~\ref{tbl:appdx:exp_table3},  we show the results of OPE with $1,500$ samples under $10$ batches using k-nearest neighbor (k-NN) regression. The upper graph shows the results with the RW policy. The lower graph shows the results with the UCB policy.

In Table~\ref{tbl:appdx:exp_table4},  we show the results of OPE with $1,500$ samples under $5$ batches using nonparametric NW regression. The upper graph shows the results with the RW policy. The lower graph shows the results with the UCB policy.

In Table~\ref{tbl:appdx:exp_table5},  we show the results of OPE with $1,500$ samples under $20$ batches using nonparametric NW regression. The upper graph shows the results with the RW policy. The lower graph shows the results with the UCB policy.

In Table~\ref{tbl:appdx:exp_table6},  we show the results of OPL with $1,500$ samples under the UCB policy. 

The best two methods are highlighted in bold.

\subsection{Details of Experiments with CyberAgent Dataset}
\label{appdx:det_exp2}
We use a logged dataset of advertisement selection. In the Cyberagent, the Thompson sampling and random sampling are simultaneously used as the behavior policies for selecting a video advertisement. An advertisement is selected as follows. First, after receiving a bid request for a user $i$ with the covariate $X_i$ in an online advertisement auction, we choose an advertisement campaign, which contains several video advertisements. Then, we choose a behavior policy from Thompson sampling and random sampling. At each period, the Thomson sampling is chosen with $5\%\sim 20\%$ probability; otherwise, the random sampling is chosen. Following the chosen behavior policy, we select a video advertisement ($A_i$) from the candidates to maximize the click rate, $Y_i$. The Thomson sampling is updated for every $30$ minute. Each batch consists of about $500\sim 1,000$ samples.

For the experiment, we create a new dataset from the original dataset. By combining some batches, we make $15$ datasets with $10,000$ samples and about $10\sim 15$ batches in each dataset. We apply the OPE estimators to the dataset generated from the Thompson sampling and estimate the policy value of the random sampling. We regard the observed result of random sampling as the true policy value. The estimation error between the estimated policy value and the observed (true) policy value is reported.

\begin{table*}[t]
\caption{Results of OPE with $1,000$ samples under $10$ batches using nonparametric NW regression. The upper graph shows the results with the RW policy. The lower graph shows the results with the UCB policy.} 
\begin{center}
\medskip
\label{tbl:appdx:exp_table1}
\scalebox{0.73}[0.73]{
\begin{tabular}{l|rr|rr|rr|rr|rr|rr}
\toprule
Datasets &  \multicolumn{2}{c|}{satimage}& \multicolumn{2}{c|}{pendigits}& \multicolumn{2}{c|}{mnist}&  \multicolumn{2}{c|}{letter}& \multicolumn{2}{c|}{sensorless}& \multicolumn{2}{c}{connect-4} \\
Metrics &      MSE &      SD &      MSE &      SD &      MSE &      SD &      MSE &      SD &      MSE &      SD &     MSE &     SD  \\
\hline
PBA2IPW &  \textbf{0.062} &  0.006 &  \textbf{0.209} &  0.196 &  \textbf{0.128} &  0.035 &  \textbf{0.265} &  0.197 &  \textbf{0.114} &  0.018 &  \textbf{0.035} &  0.035 \\
EBA2IPW &  0.065 &  0.008 &  0.232 &  0.045 &  0.244 &  0.041 &  0.399 &  0.056 &  0.208 &  0.038 &  \textbf{0.035} &  0.030 \\
EBA2IPW' &  0.065 &  0.008 &  0.212 &  0.039 &  \textbf{0.131} &  0.022 &  0.386 &  0.059 &  0.201 &  0.034 &  \textbf{0.034} &  0.032 \\
BAdaIPW &  0.161 &  0.138 &  0.277 &  0.278 &  0.180 &  0.105 &  0.284 &  0.221 &  0.135 &  0.028 &  0.046 &  0.046 \\
AdaDM &  0.163 &  0.017 &  0.500 &  0.046 &  0.496 &  0.045 &  0.438 &  0.027 &  0.380 &  0.032 &  0.145 &  0.030 \\
AIPW &  \textbf{0.046} &  0.005 &  \textbf{0.155} &  0.069 &  0.293 &  0.034 &  \textbf{0.214} &  0.118 &  \textbf{0.104} &  0.014 &  0.057 &  0.036 \\
DM &  0.101 &  0.008 &  0.450 &  0.038 &  0.339 &  0.031 &  0.421 &  0.026 &  0.357 &  0.029 &  0.090 &  0.024 \\
\bottomrule
\end{tabular}
} 
\end{center}

\begin{center}
\scalebox{0.73}[0.73]{
\begin{tabular}{l|rr|rr|rr|rr|rr|rr}
\toprule
Datasets &  \multicolumn{2}{c|}{satimage}& \multicolumn{2}{c|}{pendigits}& \multicolumn{2}{c|}{mnist}&  \multicolumn{2}{c|}{letter}& \multicolumn{2}{c|}{sensorless}& \multicolumn{2}{c}{connect-4} \\
Metrics &      MSE &      SD &      MSE &      SD &      MSE &      SD &      MSE &      SD &      MSE &      SD &     MSE &     SD \\
\hline
PBA2IPW &  0.070 &  0.015 &  0.125 &  0.042 &  \textbf{0.160} &  0.039 &  \textbf{0.331} &  0.287 &  0.257 &  0.467 &  0.053 &  0.053 \\
EBA2IPW &  \textbf{0.020} &  0.001 &  \textbf{0.058} &  0.006 &  0.320 &  0.061 &  0.426 &  0.028 &  0.268 &  0.055 &  0.035 &  0.031 \\
EBA2IPW' &  0.050 &  0.005 &  0.146 &  0.033 &  \textbf{0.152} &  0.023 &  0.414 &  0.025 &  \textbf{0.237} &  0.038 &  0.050 &  0.048 \\
BAdaIPW &  0.108 &  0.036 &  0.176 &  0.071 &  0.237 &  0.079 &  0.372 &  0.454 &  0.304 &  0.667 &  0.081 &  0.080 \\
AdaDM &  0.089 &  0.004 &  0.275 &  0.013 &  0.442 &  0.027 &  0.431 &  0.022 &  0.319 &  0.025 &  0.099 &  0.033 \\
AIPW &  0.034 &  0.002 &  \textbf{0.056} &  0.005 &  0.206 &  0.011 &  \textbf{0.275} &  0.196 &  0.246 &  0.447 &  \textbf{0.026} &  0.023 \\
DM &  \textbf{0.012} &  0.000 &  0.070 &  0.002 &  0.231 &  0.012 &  0.375 &  0.024 &  \textbf{0.211} &  0.018 &  \textbf{0.018} &  0.018 \\
\bottomrule
\end{tabular}
} 
\end{center}

\end{table*}

\begin{table*}[t]
\caption{Results of OPE with $2,000$ samples under $10$ batches using nonparametric NW regression. The upper graph shows the results with the RW policy. The lower graph shows the results with the UCB policy.} 
\begin{center}
\medskip
\label{tbl:appdx:exp_table2}
\scalebox{0.73}[0.73]{
\begin{tabular}{l|rr|rr|rr|rr|rr|rr}
\toprule
Datasets &  \multicolumn{2}{c|}{satimage}& \multicolumn{2}{c|}{pendigits}& \multicolumn{2}{c|}{mnist}&  \multicolumn{2}{c|}{letter}& \multicolumn{2}{c|}{sensorless}& \multicolumn{2}{c}{connect-4} \\
Metrics &      MSE &      SD &      MSE &      SD &      MSE &      SD &      MSE &      SD &      MSE &      SD &     MSE &     SD  \\
\hline
PBA2IPW &  \textbf{0.038} &  0.002 &  \textbf{0.147} &  0.066 &  \textbf{0.114} &  0.038 &  \textbf{0.269} &  0.218 &  \textbf{0.106} &  0.026 &  \textbf{0.019} &  0.019 \\
EBA2IPW &  0.050 &  0.006 &  0.173 &  0.023 &  0.154 &  0.019 &  0.358 &  0.058 &  0.172 &  0.021 &  \textbf{0.022} &  0.021 \\
EBA2IPW' &  0.046 &  0.006 &  0.162 &  0.021 &  \textbf{0.105} &  0.029 &  0.346 &  0.059 &  0.167 &  0.021 &  0.025 &  0.025 \\
BAdaIPW &  0.071 &  0.007 &  0.194 &  0.118 &  0.164 &  0.101 &  0.277 &  0.220 &  0.131 &  0.043 &  0.029 &  0.029 \\
AdaDM &  0.144 &  0.012 &  0.488 &  0.034 &  0.385 &  0.030 &  0.496 &  0.022 &  0.444 &  0.026 &  0.138 &  0.022 \\
AIPW &  \textbf{0.032} &  0.002 &  \textbf{0.117} &  0.026 &  0.203 &  0.015 &  \textbf{0.207} &  0.078 &  \textbf{0.101} &  0.022 &  0.054 &  0.014 \\
DM &  0.103 &  0.007 &  0.455 &  0.026 &  0.235 &  0.019 &  0.482 &  0.022 &  0.426 &  0.019 &  0.085 &  0.013 \\
\bottomrule
\end{tabular}
} 
\end{center}

\begin{center}
\scalebox{0.73}[0.73]{
\begin{tabular}{l|rr|rr|rr|rr|rr|rr}
\toprule
Datasets &  \multicolumn{2}{c|}{satimage}& \multicolumn{2}{c|}{pendigits}& \multicolumn{2}{c|}{mnist}&  \multicolumn{2}{c|}{letter}& \multicolumn{2}{c|}{sensorless}& \multicolumn{2}{c}{connect-4} \\
Metrics &      MSE &      SD &      MSE &      SD &      MSE &      SD &      MSE &      SD &      MSE &      SD &     MSE &     SD \\
\hline
PBA2IPW &  0.045 &  0.003 &  0.073 &  0.011 &  0.145 &  0.099 &  \textbf{0.282} &  0.236 &  0.120 &  0.047 &  0.038 &  0.038 \\
EBA2IPW &  \textbf{0.013} &  0.000 &  \textbf{0.030} &  0.003 &  0.143 &  0.016 &  0.435 &  0.051 &  0.192 &  0.042 &  0.024 &  0.019 \\
EBA2IPW' &  0.023 &  0.001 &  0.055 &  0.009 &  \textbf{0.091} &  0.010 &  0.404 &  0.049 &  0.159 &  0.023 &  0.027 &  0.024 \\
BAdaIPW &  0.067 &  0.008 &  0.106 &  0.030 &  0.199 &  0.105 &  0.299 &  0.249 &  \textbf{0.111} &  0.020 &  0.069 &  0.069 \\
AdaDM &  0.071 &  0.002 &  0.206 &  0.006 &  0.322 &  0.013 &  0.458 &  0.020 &  0.334 &  0.019 &  0.070 &  0.018 \\
AIPW &  0.029 &  0.001 &  \textbf{0.038} &  0.003 &  \textbf{0.088} &  0.004 &  \textbf{0.203} &  0.084 &  \textbf{0.112} &  0.049 &  \textbf{0.020} &  0.012 \\
DM &  \textbf{0.007} &  0.000 &  0.039 &  0.001 &  0.109 &  0.005 &  0.356 &  0.021 &  0.216 &  0.013 &  \textbf{0.017} &  0.008 \\
\bottomrule
\end{tabular}
} 
\end{center}

\end{table*}

\begin{table*}[t]
\caption{Results of OPE with $1,500$ samples under $10$ batches using nonparametric k-NN regression. The upper graph shows the results with the RW policy. The lower graph shows the results with the UCB policy.} 
\begin{center}
\medskip
\label{tbl:appdx:exp_table3}
\scalebox{0.73}[0.73]{
\begin{tabular}{l|rr|rr|rr|rr|rr|rr}
\toprule
Datasets &  \multicolumn{2}{c|}{satimage}& \multicolumn{2}{c|}{pendigits}& \multicolumn{2}{c|}{mnist}&  \multicolumn{2}{c|}{letter}& \multicolumn{2}{c|}{sensorless}& \multicolumn{2}{c}{connect-4} \\
Metrics &      MSE &      SD &      MSE &      SD &      MSE &      SD &      MSE &      SD &      MSE &      SD &     MSE &     SD  \\
\hline
PBA2IPW &  0.073 &  0.024 &  \textbf{0.150} &  0.075 &  \textbf{0.165} &  0.074 &  \textbf{0.205} &  0.079 &  \textbf{0.120} &  0.057 &  \textbf{0.028} &  0.028 \\
EBA2IPW &  \textbf{0.056} &  0.007 &  0.151 &  0.021 &  0.219 &  0.041 &  0.354 &  0.049 &  0.155 &  0.020 &  0.030 &  0.027 \\
EBA2IPW' &  0.061 &  0.009 &  0.181 &  0.128 &  0.202 &  0.035 &  0.341 &  0.050 &  0.146 &  0.020 &  \textbf{0.028} &  0.025 \\
BAdaIPW &  0.099 &  0.041 &  0.197 &  0.127 &  0.196 &  0.092 &  0.215 &  0.084 &  0.152 &  0.079 &  0.037 &  0.036 \\
AdaDM &  0.117 &  0.013 &  0.327 &  0.033 &  0.540 &  0.039 &  0.441 &  0.024 &  0.303 &  0.025 &  0.135 &  0.025 \\
AIPW &  \textbf{0.048} &  0.008 &  \textbf{0.108} &  0.035 &  \textbf{0.133} &  0.021 &  \textbf{0.182} &  0.046 &  \textbf{0.106} &  0.030 &  \textbf{0.021} &  0.021 \\
DM &  0.057 &  0.007 &  0.198 &  0.020 &  0.456 &  0.035 &  0.415 &  0.023 &  0.238 &  0.020 &  0.123 &  0.015 \\
\bottomrule
\end{tabular}
} 
\end{center}

\begin{center}
\scalebox{0.73}[0.73]{
\begin{tabular}{l|rr|rr|rr|rr|rr|rr}
\toprule
Datasets &  \multicolumn{2}{c|}{satimage}& \multicolumn{2}{c|}{pendigits}& \multicolumn{2}{c|}{mnist}&  \multicolumn{2}{c|}{letter}& \multicolumn{2}{c|}{sensorless}& \multicolumn{2}{c}{connect-4} \\
Metrics &      MSE &      SD &      MSE &      SD &      MSE &      SD &      MSE &      SD &      MSE &      SD &     MSE &     SD \\
\hline
PBA2IPW &  0.058 &  0.007 &  0.116 &  0.038 &  \textbf{0.149} &  0.057 &  \textbf{0.208} &  0.064 &  0.119 &  0.049 &  0.041 &  0.041 \\
EBA2IPW &  \textbf{0.014} &  0.000 &  \textbf{0.021} &  0.001 &  0.322 &  0.104 &  0.400 &  0.040 &  0.168 &  0.047 &  \textbf{0.030} &  0.022 \\
EBA2IPW' &  0.043 &  0.007 &  0.113 &  0.043 &  0.309 &  0.072 &  0.380 &  0.038 &  0.133 &  0.022 &  0.036 &  0.024 \\
BAdaIPW &  0.082 &  0.011 &  0.140 &  0.048 &  0.197 &  0.055 &  0.255 &  0.161 &  0.164 &  0.109 &  0.058 &  0.058 \\
AdaDM &  0.056 &  0.002 &  0.108 &  0.004 &  0.454 &  0.019 &  0.409 &  0.023 &  0.205 &  0.015 &  0.070 &  0.024 \\
AIPW &  0.038 &  0.004 &  0.049 &  0.011 &  \textbf{0.126} &  0.077 &  \textbf{0.168} &  0.049 &  \textbf{0.065} &  0.005 &  0.038 &  0.037 \\
DM &  \textbf{0.013} &  0.000 &  \textbf{0.048} &  0.001 &  0.264 &  0.009 &  0.301 &  0.023 &  \textbf{0.060} &  0.004 &  \textbf{0.011} &  0.011 \\
\bottomrule
\end{tabular}
} 
\end{center}

\end{table*}

\begin{table*}[t]
\caption{Results of OPE with $1,500$ samples under $5$ batches using nonparametric NW regression. The upper graph shows the results with the RW policy. The lower graph shows the results with the UCB policy.} 
\begin{center}
\medskip
\label{tbl:appdx:exp_table4}
\scalebox{0.73}[0.73]{
\begin{tabular}{l|rr|rr|rr|rr|rr|rr}
\toprule
Datasets &  \multicolumn{2}{c|}{satimage}& \multicolumn{2}{c|}{pendigits}& \multicolumn{2}{c|}{mnist}&  \multicolumn{2}{c|}{letter}& \multicolumn{2}{c|}{sensorless}& \multicolumn{2}{c}{connect-4} \\
Metrics &      MSE &      SD &      MSE &      SD &      MSE &      SD &      MSE &      SD &      MSE &      SD &     MSE &     SD  \\
\hline
PBA2IPW &  \textbf{0.046} &  0.005 &  0.127 &  0.059 &  \textbf{0.138} &  0.064 &  0.243 &  0.184 &  0.141 &  0.083 &  \textbf{0.022} &  0.022 \\
EBA2IPW &  0.049 &  0.006 &  0.123 &  0.018 &  0.140 &  0.023 &  0.241 &  0.055 &  \textbf{0.122} &  0.018 &  \textbf{0.023} &  0.022 \\
EBA2IPW' &  0.048 &  0.005 &  \textbf{0.117} &  0.017 &  \textbf{0.113} &  0.019 &  \textbf{0.233} &  0.055 &  0.136 &  0.047 &  \textbf{0.023} &  0.022 \\
BAdaIPW &  0.081 &  0.011 &  0.171 &  0.088 &  0.154 &  0.066 &  0.244 &  0.182 &  0.161 &  0.090 &  0.034 &  0.034 \\
AdaDM &  0.138 &  0.015 &  0.496 &  0.040 &  0.439 &  0.044 &  0.479 &  0.021 &  0.413 &  0.030 &  0.143 &  0.026 \\
AIPW &  \textbf{0.034} &  0.003 &  \textbf{0.100} &  0.016 &  0.249 &  0.026 &  \textbf{0.172} &  0.038 &  \textbf{0.107} &  0.050 &  0.061 &  0.023 \\
DM &  0.098 &  0.007 &  0.460 &  0.032 &  0.289 &  0.030 &  0.465 &  0.022 &  0.397 &  0.025 &  0.090 &  0.019 \\
\bottomrule
\end{tabular}
} 
\end{center}

\begin{center}
\scalebox{0.73}[0.73]{
\begin{tabular}{l|rr|rr|rr|rr|rr|rr}
\toprule
Datasets &  \multicolumn{2}{c|}{satimage}& \multicolumn{2}{c|}{pendigits}& \multicolumn{2}{c|}{mnist}&  \multicolumn{2}{c|}{letter}& \multicolumn{2}{c|}{sensorless}& \multicolumn{2}{c}{connect-4} \\
Metrics &      MSE &      SD &      MSE &      SD &      MSE &      SD &      MSE &      SD &      MSE &      SD &     MSE &     SD \\
\hline
PBA2IPW &  0.061 &  0.007 &  0.143 &  0.082 &  \textbf{0.133} &  0.045 &  \textbf{0.276} &  0.222 &  0.154 &  0.065 &  0.073 &  0.073 \\
EBA2IPW &  \textbf{0.023} &  0.001 &  \textbf{0.068} &  0.021 &  0.168 &  0.034 &  0.401 &  0.079 &  0.185 &  0.048 &  0.028 &  0.028 \\
EBA2IPW' &  0.047 &  0.011 &  0.080 &  0.022 &  \textbf{0.117} &  0.015 &  0.384 &  0.072 &  0.171 &  0.036 &  0.032 &  0.031 \\
BAdaIPW &  0.081 &  0.018 &  0.198 &  0.129 &  0.199 &  0.082 &  0.290 &  0.258 &  \textbf{0.149} &  0.042 &  0.101 &  0.100 \\
AdaDM &  0.108 &  0.006 &  0.286 &  0.012 &  0.409 &  0.021 &  0.464 &  0.019 &  0.357 &  0.026 &  0.100 &  0.030 \\
AIPW &  0.037 &  0.004 &  \textbf{0.060} &  0.009 &  0.155 &  0.008 &  \textbf{0.229} &  0.140 &  \textbf{0.127} &  0.052 &  \textbf{0.017} &  0.017 \\
DM &  \textbf{0.012} &  0.000 &  0.071 &  0.001 &  0.178 &  0.009 &  0.385 &  0.022 &  0.244 &  0.018 &  \textbf{0.012} &  0.012 \\
\bottomrule
\end{tabular}
} 
\end{center}

\end{table*}

\begin{table*}[t]
\caption{Results of OPE with $1,500$ samples under $20$ batches using nonparametric NW regression. The upper graph shows the results with the RW policy. The lower graph shows the results with the UCB policy.} 
\begin{center}
\medskip
\label{tbl:appdx:exp_table5}
\scalebox{0.73}[0.73]{
\begin{tabular}{l|rr|rr|rr|rr|rr|rr}
\toprule
Datasets &  \multicolumn{2}{c|}{satimage}& \multicolumn{2}{c|}{pendigits}& \multicolumn{2}{c|}{mnist}&  \multicolumn{2}{c|}{letter}& \multicolumn{2}{c|}{sensorless}& \multicolumn{2}{c}{connect-4} \\
Metrics &      MSE &      SD &      MSE &      SD &      MSE &      SD &      MSE &      SD &      MSE &      SD &     MSE &     SD  \\
\hline
PBA2IPW &  \textbf{0.045} &  0.004 &  \textbf{0.139} &  0.065 &  \textbf{0.163} &  0.077 &  \textbf{0.376} &  0.328 &  \textbf{0.162} &  0.076 &  \textbf{0.058} &  0.058 \\
EBA2IPW &  \textbf{0.051} &  0.005 &  0.274 &  0.035 &  0.248 &  0.029 &  0.467 &  0.044 &  0.265 &  0.044 &  \textbf{0.032} &  0.024 \\
EBA2IPW' &  \textbf{0.051} &  0.006 &  0.267 &  0.035 &  \textbf{0.162} &  0.077 &  0.460 &  0.041 &  0.258 &  0.037 &  \textbf{0.032} &  0.031 \\
BAdaIPW &  0.089 &  0.013 &  0.193 &  0.089 &  0.213 &  0.103 &  0.399 &  0.396 &  0.188 &  0.100 &  0.076 &  0.075 \\
AdaDM &  0.144 &  0.012 &  0.492 &  0.033 &  0.434 &  0.037 &  0.476 &  0.022 &  0.412 &  0.027 &  0.139 &  0.023 \\
AIPW &  0.057 &  0.021 &  \textbf{0.109} &  0.024 &  0.234 &  0.018 &  \textbf{0.282} &  0.127 &  \textbf{0.150} &  0.062 &  0.073 &  0.057 \\
DM &  0.101 &  0.003 &  0.447 &  0.024 &  0.265 &  0.022 &  0.457 &  0.021 &  0.389 &  0.021 &  0.083 &  0.015 \\
\bottomrule
\end{tabular}
} 
\end{center}

\begin{center}
\scalebox{0.73}[0.73]{
\begin{tabular}{l|rr|rr|rr|rr|rr|rr}
\toprule
Datasets &  \multicolumn{2}{c|}{satimage}& \multicolumn{2}{c|}{pendigits}& \multicolumn{2}{c|}{mnist}&  \multicolumn{2}{c|}{letter}& \multicolumn{2}{c|}{sensorless}& \multicolumn{2}{c}{connect-4} \\
Metrics &      MSE &      SD &      MSE &      SD &      MSE &      SD &      MSE &      SD &      MSE &      SD &     MSE &     SD \\
\hline
PBA2IPW &  0.052 &  0.006 &  0.089 &  0.019 &  \textbf{0.131} &  0.053 &  0.372 &  1.023 &  \textbf{0.117} &  0.057 &  0.037 &  0.037 \\
EBA2IPW &  \textbf{0.013} &  0.000 &  \textbf{0.021} &  0.001 &  0.223 &  0.015 &  0.446 &  0.037 &  0.194 &  0.025 &  0.030 &  0.017 \\
EBA2IPW' &  0.035 &  0.002 &  0.099 &  0.013 &  0.140 &  0.065 &  0.428 &  0.028 &  0.206 &  0.028 &  0.035 &  0.028 \\
BAdaIPW &  0.071 &  0.011 &  0.125 &  0.033 &  0.194 &  0.091 &  0.387 &  1.065 &  0.138 &  0.071 &  0.055 &  0.055 \\
AdaDM &  0.054 &  0.002 &  0.197 &  0.006 &  0.350 &  0.016 &  0.444 &  0.023 &  0.310 &  0.025 &  0.066 &  0.021 \\
AIPW &  0.034 &  0.003 &  0.044 &  0.004 &  \textbf{0.131} &  0.006 &  \textbf{0.306} &  0.593 &  \textbf{0.090} &  0.022 &  \textbf{0.024} &  0.012 \\
DM &  \textbf{0.011} &  0.000 &  \textbf{0.041} &  0.001 &  0.151 &  0.007 &  \textbf{0.359} &  0.023 &  0.196 &  0.016 &  \textbf{0.020} &  0.010 \\
\bottomrule
\end{tabular}
} 
\end{center}

\end{table*}

\begin{table*}[t]
\begin{center}
\caption{Results of OPL under the UCB policy. We highlight in bold the best two estimators in each dataset.} 
\medskip
\label{tbl:appdx:exp_table6}
\scalebox{0.73}[0.73]{
\begin{tabular}{l|rr|rr|rr|rr|rr|rr}
\toprule
Datasets &  \multicolumn{2}{c|}{satimage}& \multicolumn{2}{c|}{pendigits}& \multicolumn{2}{c|}{mnist}& \multicolumn{2}{c|}{letter}& \multicolumn{2}{c|}{sensorless}& \multicolumn{2}{c}{connect-4} \\
Metrics &      RWD &      SD &      RWD &      SD &      RWD &      SD &      RWD &      SD &      RWD &      SD &     RWD &     SD \\
\hline
PBA2IPW &  0.806 &  0.024 &  \textbf{0.656} &  0.145 &  \textbf{0.743} &  0.028 &  \textbf{0.156} &  0.048 &  \textbf{0.317} &  0.062 &  0.645 &  0.030 \\
EBA2IPW &  0.810 &  0.026 &  0.463 &  0.100 &  0.653 &  0.054 &  0.099 &  0.046 &  \textbf{0.312} &  0.075 &  \textbf{0.668} &  0.030 \\
BAdaIPW&  \textbf{0.811} &  0.021 &  \textbf{0.511} &  0.107 &  \textbf{0.656} &  0.070 &  \textbf{0.126} &  0.038 &  0.275 &  0.071 &  0.665 &  0.024 \\
AdaDM &  0.786 &  0.029 &  0.407 &  0.053 &  0.191 &  0.078 &  0.055 &  0.023 &  0.207 &  0.039 &  0.653 &  0.030 \\
AIPW &  \textbf{0.819} &  0.015 &  0.419 &  0.099 &  0.644 &  0.101 &  0.120 &  0.040 &  0.290 &  0.057 &  \textbf{0.666} &  0.028 \\
DM &  0.798 &  0.026 &  0.387 &  0.035 &  0.212 &  0.080 &  0.089 &  0.036 &  0.220 &  0.045 &  0.653 &  0.030 \\
\bottomrule
\end{tabular}
} 
\end{center}
\vspace{-0.5cm}
\end{table*}

\end{document}